\documentclass[twoside]{article}

\usepackage[preprint]{aistats2027}
\usepackage{graphicx}
\usepackage{booktabs}
\usepackage{natbib}
\usepackage{amsfonts,amssymb,amsthm,mathtools,mathrsfs,amsmath}
\usepackage{xfrac,hyperref}
\usepackage[table]{xcolor}
\usepackage{tabularx}
\usepackage{longtable}
\usepackage{xcolor}
\usepackage{adjustbox}
\usepackage{array}
\usepackage{bbm}
\usepackage{placeins}
\theoremstyle{plain}
\newtheorem{theorem}{Theorem}
\newtheorem{lemma}{Lemma}

\theoremstyle{definition}
\newtheorem{definition}{Definition}

\theoremstyle{remark}

\newcommand{\sref}[1]{\S\ref{#1}}

\begin{document}

\twocolumn[

\aistatstitle {Divergence based Ranked Preference Alignment for LLMs}

\aistatsauthor{ Rajdeep Haldar \And Yue Xing \And  Qifan Song }

\aistatsaddress{Purdue University \And  Michigan State University \And Purdue University } ]


\begin{abstract}
Alignment of large language models typically relies on pairwise preference feedback, though ranked feedback can capture more comprehensive information across multiple responses within a single annotation.
However, ranked-preference alignment is underexplored in practice compared to simpler pairwise approaches, potentially resulting in suboptimal behavior. Concurrently, recent research has conceptualized large language model alignment objectives as divergence estimators between aligned (preferred) and unaligned (less-preferred) response distributions, resulting in a family of objectives with an interpretable theoretical foundation and strong empirical performance. This work unifies these perspectives by introducing $f$-RPO, a $f$-divergence objective for ranked preference optimization that extends divergence-based alignment to ranked feedback and recovers pairwise preference alignment as a special case. Our theoretical analysis establishes rank-order consistency, showing that $f$-RPO preserves the underlying preference ordering at convergence.
Empirical results indicate that total-variation and reverse-KL variants of the proposed objective consistently outperform ranked-alignment baselines in both ranking and generation quality across Qwen, Llama, \& Gemma models ranging from 0.5B to 8B parameters.
\end{abstract}

\section{Introduction}



Aligning large language models (LLMs) with human preferences is central to modern post-training, enabling pretrained models
to follow user intent, produce helpful responses, and satisfy
desired behavioral and safety requirements
\citep{ouyang2022training,bai2022constitutional}.
Current preference-alignment methods, however, predominantly
operate on human annotated \emph{pairwise} feedback, where one response is labeled as
preferred over another
\citep{rafailov2023direct,ethayarajh2024kto,meng2024simpo}.
Although pairwise comparisons provide a simple and effective
supervision signal, ranked feedback provides substantially more information within a single annotation by expressing relative preferences among several responses.
A naive approach is to decompose a ranking of $k$ responses into its
$\binom{k}{2}$ constituent pairwise comparisons. Such a reduction,
however, increases the number of training comparisons and discards the
structured nature of the original ranking. Statistically, one can interpret pairwise methods to assume a composite likelihood while the latter utilizes an informative joint likelihoods, thus the former one is inherently less efficient due to its ignorance of within-ranking information.

Ranked alignment methods therefore aim to extend binary preference optimization to ranked feedback
Mallows-based ranked preference optimization
\citep{yuan2023rrhf,song2024preference,tang2026beyond}.
These results demonstrate the promise of ranked preference alignment, but the setting remains considerably less explored than its convenient pairwise
counterpart.

In parallel, recent work has advanced a conceptual view of preference alignment: alignment objectives can be
interpreted as divergence estimators between an
\emph{aligned} distribution, represented by preferred responses, and
an \emph{unaligned} distribution, represented by less-preferred
responses
\citep{haldar2025llm,haldar2026f}.
This perspective provides a unified interpretation of existing
alignment objectives, yields strong theoretical guarantees, and naturally proposes new objectives by
varying the choice of divergence. Existing formulations, however,
are limited for binary preference feedback.
\begin{figure*}[!ht]
    \centering
    \includegraphics[width=\linewidth, trim= 0 1.1cm 0 0.5cm, clip]{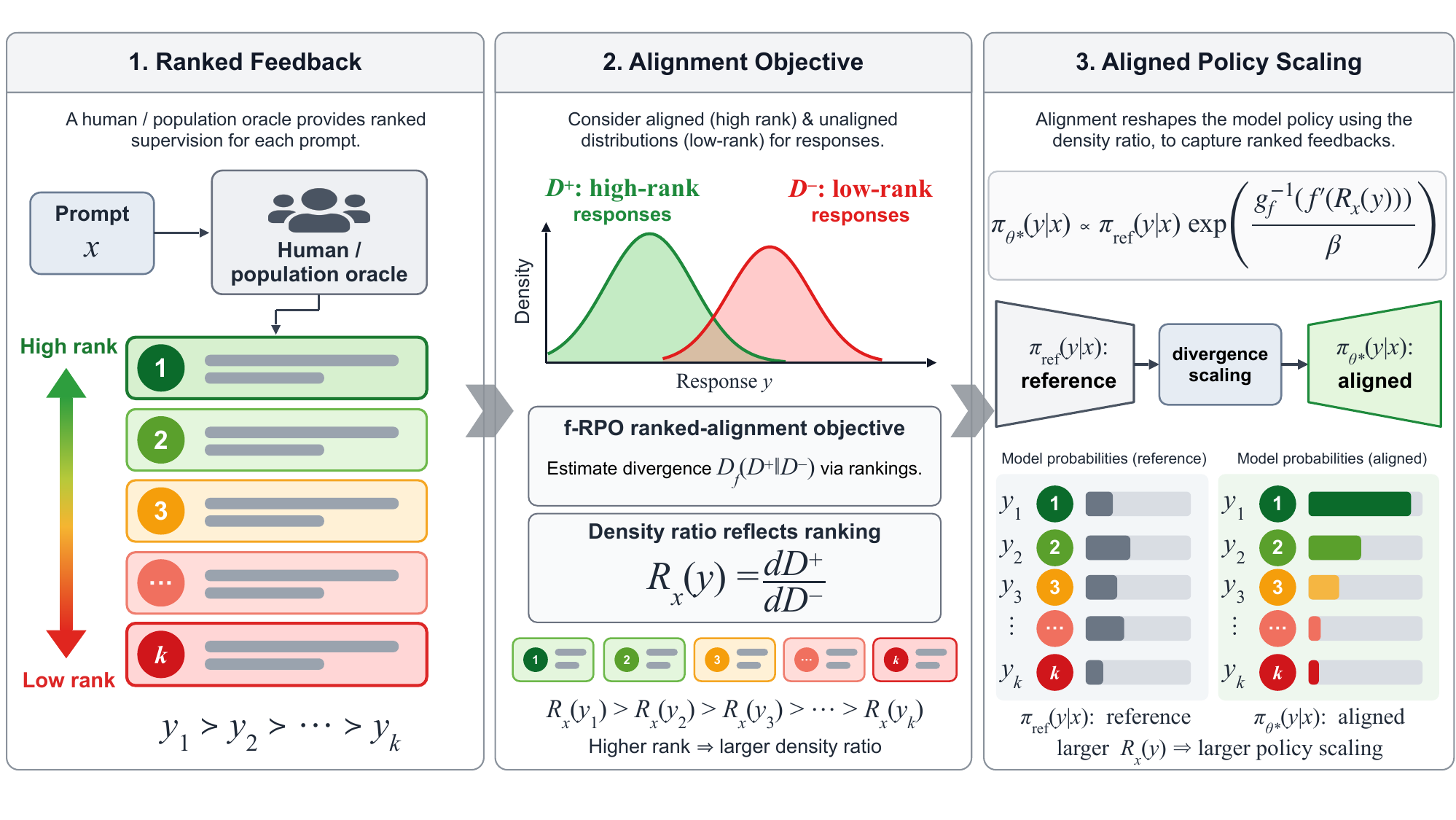}
    \caption{Overview of $f$-RPO. Ranked feedback implicitly induces aligned and unaligned response distributions. Alignment then scales the reference policy to reflect ranking.}
    \label{fig:framework}
\end{figure*}
In this work, we extend the divergence-based view of alignment to
ranked preferences and introduce \emph{ranked $f$-divergence
preference optimization} ($f$-RPO), a family of alignment objectives
for learning directly from ordered response lists. Our construction
generalizes the notions of aligned and unaligned response
distributions to ranked feedback: the ranked aligned distribution
emphasizes responses according to their role as higher-ranked
alternatives, whereas the ranked unaligned distribution emphasizes
their role as lower-ranked alternatives. Applying variational
$f$-divergence estimation to these distributions yields a family of
ranked alignment objectives indexed by the choice of $f$-divergence.
Figure~\ref{fig:framework} provides an overview of the framework.

We summarize our contributions as follows:
\paragraph{(1) Divergence Framework} We view alignment on ranked feedback as a divergence estimation problem between two contrasting distributions over prompt-response pairs. A ranked aligned distribution which favors sampling of high-rank responses vs a ranked unaligned one which favors sampling of low-rank responses for a given prompt. This view and the corresponding results generalize the recent divergence view established for binary preference alignment. (Refer \sref{sec: ranked_fdivergence optimiazation})
\paragraph{(2) New Ranked Alignment Objectives} Based on this framework, we propose a class of ranked alignment objectives, $f$-RPO, based on the choice of $f$-divergence. We show that theoretically, this family of objectives leads to an optimal policy that scales the reference policy to reflect ranking. Moreover, under mild conditions, they can preserve the true ranking of responses. (Refer \sref{sec:frpo-objective}, \sref{sec: theory})
\paragraph{(3) Experiments} Across Qwen, Llama, and Gemma models spanning multiple parameter
scales, reverse-KL and total-variation instantiations of $f$-RPO
consistently outperform strong ranked-alignment baselines in both
ranking quality and downstream generation quality \sref{sec: experiments}.
In addition to generation-based evaluation, we directly measure how
well aligned models recover the target response ordering using ranking
metrics such as Spearman correlation and top-response retrieval
accuracy, providing a complementary evaluation of ranked preference
alignment.
\section{Related Works}
\paragraph{Ranked/listwise alignment}
When more than two candidate responses are ranked for the same prompt, it can provide more preference information, potentially leading to better model alignment than pairwise preference alignment. However, it also presents new challenges, such as how to handle ties \citep{henderson2025modelling}, inconsistent ranking of LLM judges \citep{shi2025judging} and Condorcet cycle \citep{liu2026statistical}.
\paragraph{Preference data collection} 
The quality of preference data plays a crucial role in reward model training and LLM alignment. While the community has put a tremendous effort into curating large-scale open-source reference data sets \citep[e.g.,][]{kopf2023openassistant,ji2023beavertails},
AI feedback offers a cheap alternative to human annotation 
\cite{bai2022constitutionalaiharmlessnessai, cui2024ultra}.
Active preference learning frames preference data collection as a sequential experiment design problem \citep{Muldrew24active,xu2025rlthf}: given the current reward model estimation, it selects the prompt and response pair that can deliver the greatest information about the unknown preference function.  preference

\section{Background}
\label{sec:background}
\subsection{Preference Alignment}

Given a prompt $x$ and two responses $y_w$ and $y_l$, a preference
annotation indicates that $y_w \succ y_l$. Reinforcement Learning
with Human Feedback (RLHF)~\citep{ouyang2022training} first learns a
reward model $r_\phi(x,y)$, commonly using the Bradley--Terry model,
\[
    P(y_w\succ y_l)
    =
    \sigma\!\left(
        r_\phi(x,y_w)-r_\phi(x,y_l)
    \right),
\]
and subsequently optimizes the policy against this learned reward via reinforcement learning techniques.

\paragraph{Direct alignment and implicit reward.}
Direct preference optimization methods avoid explicitly training a
reward model and instead optimize the policy directly from preference
data. DPO~\citep{rafailov2024direct} uses the implicit reward
\begin{equation}
    r_\theta(x,y)
    =
    \beta
    \log
    \frac{\pi_\theta(y\mid x)}
         {\pi_{\mathrm{ref}}(y\mid x)},
    \label{eq:background-implicit-reward}
\end{equation}
where $\pi_{\mathrm{ref}}$ is a fixed reference policy and
$\beta>0$ controls the reward scale. The DPO objective is
\[
    \mathcal L_{\mathrm{DPO}}(\theta)
    =
    -\mathbb E
    \log
    \sigma\!\left(
        r_\theta(x,y_w)-r_\theta(x,y_l)
    \right).
\]
Several direct-alignment methods, including KTO
\citep{ethayarajh2024kto}, BCO~\citep{jung2024binary}, and
KLDO~\citep{haldar2025llm}, similarly operate through this
policy-relative reward parameterization.

\paragraph{Divergence-based alignment.}
Recent work provides a distributional view of alignment,
interpreting alignment objectives as variational $f$-divergence
estimation between an aligned response distribution $\mathcal D^+$
and an unaligned response distribution $\mathcal D^-$
\citep{haldar2025llm,haldar2026f}. For a $f$-divergence, the generic policy objective takes the form
\begin{equation}
\begin{aligned}
    \mathcal L_f(\theta)
    ={}&
    -\mathbb E_{(x,y)\sim\mathcal D^+}
    \left[g_f(r_\theta(x,y))\right]
    \\
    &+
    \mathbb E_{(x,y)\sim\mathcal D^-}
    \left[
        f^*\!\left(g_f(r_\theta(x,y))\right)
    \right],
\end{aligned}
\label{eq:background-divergence-alignment}
\end{equation}
where $f^*, g_f$ are the convex conjugate of $f$ and a monotone link function that maps the implicit reward to the effective domain of the variational critic. The above objective estimates the $f$-divergence or separation of the distributions at convergence, i.e., $\mathcal -L_f(\theta^*)= D_f(\mathcal{D^+}\|\mathcal{D^-})=\inf_g \mathbb E_{(x,y)\sim\mathcal D^+}g(x,y)+E_{(x,y)\sim\mathcal D^-}f^*\circ g(x,y)$. Different choices of $f$
induce different alignment objectives, while the construction of
$\mathcal D^+$ and $\mathcal D^-$ determines what constitutes aligned
and unaligned feedback. For pairwise preference alignment, these
distributions correspond to preferred and rejected responses. Moreover, direct alignment methods like KTO, BCO, and KLDO are special cases of the above objective with total variation, Jensen-Shannon, and KL divergences, respectively.
Additional background on variational $f$-divergence estimation,
canonical link functions, and existing divergence-based objectives is
provided in Appendix~\ref{app:divergence-alignment}.

\subsection{Ranked Preference Alignment}

Beyond pairwise feedback, a ranking over multiple candidate responses
contains additional relative preference information. A common approach
is to extend probabilistic ranking models to preference alignment. For
example, under the Plackett--Luce model,
\begin{equation}
    P\!\left(
        y_{(1)}\succ\cdots\succ y_{(k)}
    \right)
    =
    \prod_{i=1}^{k-1}
    \frac{
        \exp(r_\phi(x,y_{(i)}))
    }{
        \sum_{j=i}^{k}
        \exp(r_\phi(x,y_{(j)}))
    }.
    \label{eq:pl-ranked-preference}
\end{equation}
As in DPO, one can obtain direct alignment objectives by replacing the
explicit reward model in Eq.~\eqref{eq:pl-ranked-preference} with the implicit reward ~\eqref{eq:background-implicit-reward}. RCPO
\citep{tang2026beyond} generalizes this idea to broader probabilistic
rank-choice models, including Mallows-type models, while
RRHF~\citep{yuan2023rrhf} directly penalizes disagreement between the
observed ranking and model likelihoods.

Our goal is to combine these two perspectives: retain the richer
supervision of ranked feedback while extending the divergence-based
alignment framework beyond pairwise preferences.

\section{Ranked $f$-Divergence Preference Optimization}
\label{sec: ranked_fdivergence optimiazation}
\subsection{Ranked Preference Model}
\label{sec:framework}
Let $\mathcal X$ and $\mathcal Y$ denote the prompt and response spaces,
respectively. For each prompt $x\sim p_{\mathcal X}$, consider $k\ge 2$
candidate responses $y_1,\ldots,y_k$ sampled independently from a response
distribution $\mathcal D_x$, given prompt $x$ with corresponding density $p_x(y)$.
We assume the existence of an oracle utility function
$u_x:\mathcal Y\to\mathbb R$ that determines the preference order \citep{thurstone1994law,plackett1975analysis}, i.e.,
\begin{equation}
    y\succ_x y'
    \iff
    u_x(y)>u_x(y').
    \label{eq:utility-ranking}
\end{equation}
Sorting the candidates in decreasing order of utility yields
$y_{(1)}\succ_x\cdots\succ_x y_{(k)}$.

Let $F_x(t)=\Pr_{y\sim\mathcal D_x}(u_x(y)\le t)$ denote the CDF of
the utility under the base response distribution. Assuming a continuous
utility distribution, the response occupying rank $i$ follows a distribution $\mathcal D_x^i$, derived from the ordered statistics of responses, with density:
\begin{equation}
\begin{aligned}
p_x^i(y)
&=
\frac{k!}{(i-1)!(k-i)!}\,p_x(y)
\\[-1mm]
&\quad\times
F_x(u_x(y))^{k-i}
\bigl[1-F_x(u_x(y))\bigr]^{i-1}.
\end{aligned}
\label{eq:rank-conditional-density}
\end{equation}
The above equation is a standard ordered-statistic density (with details provided
in Appendix~\ref{app:rank-conditional-density}).

\subsection{Ranked Aligned and Unaligned Distributions}
\label{sec:ranked-distributions}

Building on the divergence-based alignment formulation in
Section~\ref{sec:background}, we construct aligned and unaligned
response distributions from ranked feedback.

We first consider the augmented space
\[
    \mathcal Z
    =
    \mathcal X\times\{1,\ldots,k-1\}\times\mathcal Y.
\]
For a rank index $i$, the response occupying rank $i$ is preferred
relative to any response at a lower rank $j>i$. Let $\rho_x(i)$ be a
distribution over rank indices $i\in\{1,\ldots,k-1\}$, and let
$q_x(j\mid i)$ be a distribution over lower ranks
$j\in\{i+1,\ldots,k\}$. 
$\rho_x$ and $q_x$ are user-specified distributions, guiding the construction of aligned and unaligned response distributions. 
We define $\widetilde{\mathcal D}^{\pm}$ with densities:
\begin{align}
    d\widetilde{\mathcal D}^{+}(x,i,y)
    &=
    p_{\mathcal X}(x)\,
    \rho_x(i)\,p_x^i(y),
    \label{eq:rank-aligned-augmented}
    \\
    d\widetilde{\mathcal D}^{-}(x,i,y)
    &=
    p_{\mathcal X}(x)\,
    \rho_x(i)
    \sum_{j=i+1}^{k}
    q_x(j\mid i)\,p_x^j(y).
    \label{eq:rank-unaligned-augmented}
\end{align}
Thus, conditioned on $i$, $\widetilde{\mathcal D}^{+}$ draws the
response at rank $i$, whereas $\widetilde{\mathcal D}^{-}$ draws a
lower-ranked alternative according to $q_x(\cdot\mid i)$.

Our policy $\pi_\theta$ depends only on $(x,y)$, so we marginalize $\widetilde{\mathcal D}^{\pm}$ over the rank index to obtain the corresponding aligned/unaligned distribution of interest $\mathcal{D}^\pm$
\begin{align}
    d\mathcal D^{+}(x,y)
    &=
    \sum_{i=1}^{k-1}
    \widetilde{\mathcal D}^{+}(x,i,y)
    =
    p_{\mathcal X}(x)
    \sum_{i=1}^{k-1}
    \rho_x(i)p_x^i(y),
    \label{eq:rank-aligned-dist}
    \\
    d\mathcal D^{-}(x,y)
    &=
    \sum_{i=1}^{k-1}
    \widetilde{\mathcal D}^{-}(x,i,y)
    \nonumber\\
    &=
    p_{\mathcal X}(x)
    \sum_{i=1}^{k-1}
    \rho_x(i)
    \sum_{j=i+1}^{k}
    q_x(j\mid i)p_x^j(y).
    \label{eq:rank-unaligned-dist}
\end{align}

Hence, $\mathcal D^{+}$ represents responses sampled according to their
role as preferred responses in the ranking, while $\mathcal D^{-}$
represents responses sampled according to their role as lower-ranked
alternatives. The choices of $\rho_x$ and $q_x$ determine how different
rank comparisons contribute to these two distributions.
\subsection{$f$-RPO Objective}
\label{sec:frpo-objective}

We now apply the divergence-based alignment objective in
Eq.~\eqref{eq:background-divergence-alignment} to the ranked aligned
and unaligned distributions $\mathcal D^+$ and $\mathcal D^-$ defined
above. Using the implicit reward from
Eq.~\eqref{eq:background-implicit-reward}, we obtain our
\emph{ranked $f$-divergence preference optimization} ($f$-RPO)
objective:
\begin{equation}
\begin{split}
    \mathcal L_{f\text{-RPO}}(\theta)
    =
    -&\mathbb E_{x,\mathbf y_{(1:k)}}
    \sum_{i=1}^{k-1}\rho_x(i)
    \Bigg[
        g_f(r_{\theta,i})
        \\
        &\qquad -
        \sum_{j=i+1}^{k}
        q_x(j\mid i)
        f^*\!\left(g_f(r_{\theta,j})\right)
    \Bigg],
    \label{eq:frpo-loss}
\end{split}
\end{equation}
where
$\mathbf y_{(1:k)}=(y_{(1)},\ldots,y_{(k)})$
are the ranked responses and
$r_{\theta,i}:=r_\theta(x,y_{(i)})$ is the implicit reward Eq. \eqref{eq:background-implicit-reward}.

\begin{table}[!ht]
\centering
\caption{
Realization of $f$-RPO objective for TV and RKL divergence.
}
\label{tab:main-frpo-objectives}
\small
\setlength{\tabcolsep}{4pt}

\resizebox{\columnwidth}{!}{%
\begin{tabular}{lc}
\toprule
Divergence & $\mathcal L_{f\text{-RPO}}(\theta)$ \\
\midrule

Reverse KL
&
$\displaystyle
\mathbb E_{x,\mathbf y_{(1:k)}}
\sum_{i=1}^{k-1}\rho_x(i)
\left[
e^{-r_{\theta,i}}
+
\sum_{j=i+1}^{k}
q_x(j\mid i)\,r_{\theta,j}
\right]
$
\\[3mm]

Total variation
&
$\displaystyle
\frac{1}{2}
\mathbb E_{x,\mathbf y_{(1:k)}}
\sum_{i=1}^{k-1}\rho_x(i)
\left[
-\sigma(r_{\theta,i})
+
\sum_{j=i+1}^{k}
q_x(j\mid i)\sigma(r_{\theta,j})
\right]
$
\\

\bottomrule
\end{tabular}%
}
\end{table}
The first term rewards responses selected from the preferred ranks,
while the second penalizes responses sampled from their lower-ranked
alternatives. The weighting
distributions $\rho_x$ and $q_x$ determine how the ranked feedback
contributes to the aligned and unaligned terms of the divergence
objective. In our implementation we choose $\rho_x(i)\propto k-i$ and
$q_x(j\mid i)\propto j-i$, this choice will be justified in \sref{sec: theory}, Theorem \ref{thm:rank-consistency}. Table \ref{tab:main-frpo-objectives} provides objectives for total variation (TV) and reverse-KL (RKL) divergence which show the best empirical performances. Alignment objectives for other divergences and their empirical performance w.r.t. to TV, RKL are provided in Table \ref{tab:frpo-all-objectives} and \sref{sec: fdiv ablation}

\section{Theoretical Properties}
\label{sec: theory}

\begin{definition}[Density Ratio]
\label{defn: density_ratio}
    Let $R_x(y)=\frac{d\mathcal D^+}{d\mathcal D^-}(x,y)$ be the aligned-to-unaligned density ratio. Where $d\mathcal D^\pm$ are defined in  Eqns. \eqref{eq:rank-aligned-dist}, \eqref{eq:rank-unaligned-dist}.
\end{definition}
This quantity can be interpreted as how likely a response $y$ belongs to a higher rank (aligned) rather than a lower rank (unaligned), for a prompt $x$.
\begin{definition}[Rank favored weights]
\label{defn: rank_favor_wts}
    Let the aligned and unaligned masses associated with rank $i$ in Eqns. \eqref{eq:rank-aligned-dist}, \eqref{eq:rank-unaligned-dist} be
\begin{equation}
    B_i
    =
    \rho_x(i)\mathbf{1}\{i<k\},
    \qquad
    A_i
    =
    \sum_{\ell<i}
    \rho_x(\ell)q_x(i\mid\ell).
    \label{eq:rank-masses}
\end{equation}
The aligned/unaligned distributions $\mathcal{D}^\pm$ are called to have \emph{rank favored weights} iff 
\begin{equation}
    B_iA_j\ge B_jA_i,
    \qquad i<j.
    \label{eq:rank-weight-condition}
\end{equation}
\end{definition}
Here, $B_i$ is the probability that rank $i$ contributes as a
preferred response, while $A_i$ is its aggregate probability of
contributing as a lower-ranked alternative. The ranked favored weight condition essentially enforces the contribution of higher ranks contributes more to being a preferred resonse than a lower ranked one ($B_i/B_j\geq A_i/A_j$ for non-zero denominators).
The choice of $\rho_x(i)\propto k-i$ and
$q_x(j\mid i)\propto j-i$ enforce the above condition (\sref{app:rank-weight-condition}). 
\begin{lemma}[Density ratio reflects utility]
\label{lem:density_ratio_util}
Suppose the aligned and unaligned distribution has rank favored weights (Defn. \ref{defn: rank_favor_wts}), then $u_x(y)>u_x(y')$ implies that $R_x(y)\geq R_x(y')$, for all $y,y'$ for a given prompt $x$. 
\end{lemma}
The above lemma says that the density ratio reflects the true ranking or underlying utility if the (un) aligned distributions have ranked favored weights.
\subsection{Divergence Estimation and Alignment}
\label{sec:divergence-estimation}
In our next theorem we express the optimal policy post-alignment in terms of this ratio. Moreover, theorem \ref{thm:rank-consistency} shows that

\begin{theorem}[Divergence estimation and optimal policy]
\label{thm:divergence-alignment}
The ranked alignment objective $\mathcal L_{f\text{-RPO}}(\theta)$ in Eq. \eqref{eq:frpo-loss} satisfies the following at optimality:
\begin{equation}
    -\mathcal L_{f\text{-RPO}}(\theta^*)
    =
    D_f(\mathcal D^+\|\mathcal D^-),
    \label{eq:frpo-divergence-optimum}
\end{equation}
with optimal policy
\begin{equation}
    \pi_{\theta^*}(y\mid x)      
    \propto
    \pi_{\mathrm{ref}}(y\mid x)\cdot\exp\left(
        \frac{g_f^{-1}(f'(R_x(y)))}{\beta}
    \right)
    \label{eq:optimal-frpo-policy}
\end{equation}
where, $\theta^*=\arg\inf\limits_\theta \mathcal L_{f\text{-RPO}}(\theta)$.
\end{theorem}

Theorem~\ref{thm:divergence-alignment} gives two complementary
interpretations of $f$-RPO: the objective estimates the divergence
between the ranked aligned and unaligned distributions, while the
optimal policy reweights the reference policy according to the
aligned-to-unaligned density ratio. Since $f'$ is nondecreasing for
convex $f$ and $g_f$ is monotone (Table \ref{tab:frpo-all-objectives}), larger values of $R_x(y)$ receive larger implicit rewards.
\subsection{Rank-Order Consistency}
\label{sec:rank-consistency}

We next show that the optimal solution of
Eq.~\eqref{eq:optimal-frpo-policy} is \emph{rank-order consistent}.
In particular, the implicit reward preserves the underlying utility
ordering, and the resulting policy preserves the same ordering under
a sufficiently small reference-policy bias ($\beta$).

\begin{theorem}[Rank-order consistency]
\label{thm:rank-consistency}
Suppose the aligned and unaligned distributions
$\mathcal D^\pm$ have rank-favored weights
(Definition~\ref{defn: rank_favor_wts}) and the conditions of
Theorem~\ref{thm:divergence-alignment} hold.

First, the optimal implicit reward preserves the underlying utility
ordering:
\begin{equation}
    u_x(y)>u_x(y')
    \quad\Longrightarrow\quad
    r_{\theta^*}(x,y)
    \ge
    r_{\theta^*}(x,y').
    \label{eq:utility-reward-order}
\end{equation}
Consequently, for any observed ranking,
\begin{equation}
    r_{\theta^*}(x,y_{(1)})
    \ge
    r_{\theta^*}(x,y_{(2)})
    \ge
    \cdots
    \ge
    r_{\theta^*}(x,y_{(k)}).
    \label{eq:frpo-rank-consistency}
\end{equation}

Moreover, suppose
$\pi_{\mathrm{ref}}(y_{(i)}\mid x)>0$
for all prompts $x$ and ranks $i$. For adjacent ranks, define
\[
    \Delta_{x,i}
    :=
    g_f^{-1}\!\left(f'(R_x(y_{(i)}))\right)
    -
    g_f^{-1}\!\left(f'(R_x(y_{(i+1)}))\right),
\]
and
\[
    b_{x,i}
    :=
    \log
    \frac{
        \pi_{\mathrm{ref}}(y_{(i+1)}\mid x)
    }{
        \pi_{\mathrm{ref}}(y_{(i)}\mid x)
    }.
\]
Then the optimal policy preserves the ranked order,
\[
    \pi_{\theta^*}(y_{(1)}\mid x)
    \ge
    \cdots
    \ge
    \pi_{\theta^*}(y_{(k)}\mid x),
    \qquad \forall x,
\]
whenever
\begin{equation}
    \beta
    \le
    \inf_{x,\;i:\,b_{x,i}>0}
    \frac{\Delta_{x,i}}{b_{x,i}},
    \label{eq:beta-rank-condition}
\end{equation}
provided the right-hand side is positive.
If $b_{x,i}\le0$ for all $(x,i)$, the policy ordering holds for any
$\beta>0$.
\end{theorem}

The proof is provided in \sref{app:proof-rank-consistency}.
In our experiments, we use
$\rho_x(i)\propto k-i$ and
$q_x(j\mid i)\propto j-i$.
The former accounts for the number of lower-ranked alternatives
available to rank $i$, while the latter places greater weight on
comparisons with larger rank separation. These weights satisfy
\emph{rank favored} condition Eq.~\eqref{eq:rank-weight-condition}.

Theorem~\ref{thm:rank-consistency} also quantifies how the
regularization parameter $\beta$ controls the influence of reference
policy bias. If the reference policy already respects the desired
ranking or is uniform across all responses, i.e., $b_{x,i}\le0$, then the aligned policy preserves the ranking
for any $\beta>0$. More generally, whenever the alignment-induced
score gaps $\Delta_{x,i}$ are positive, a sufficiently small $\beta$
can overcome any finite opposing bias from the reference policy.
\begin{figure*}[!ht]
  \centering
  \includegraphics[width=\textwidth, trim= 2cm 0cm 0cm 0cm, clip]{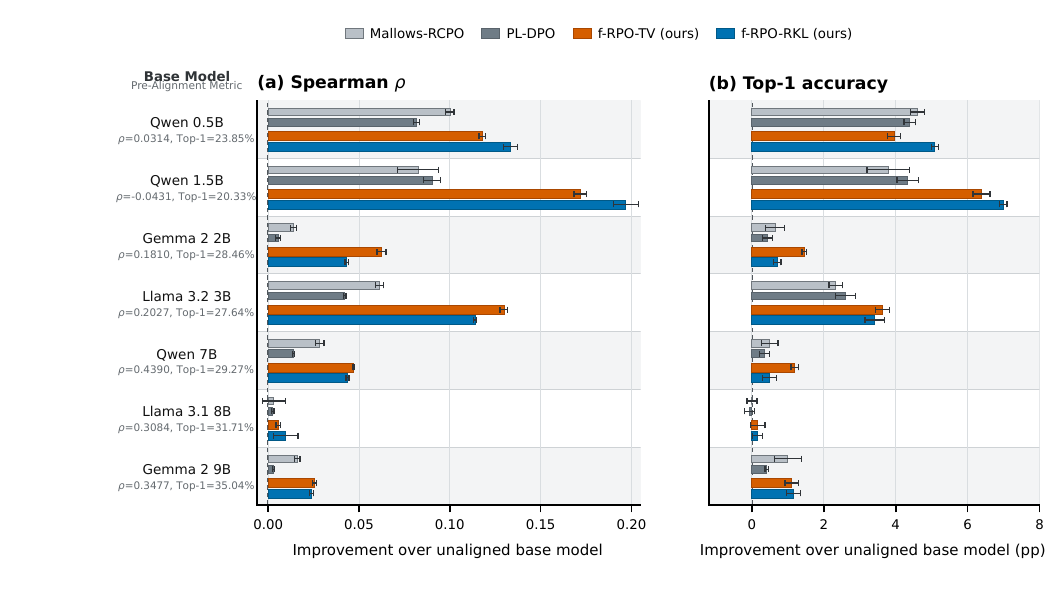}
  \caption{comparison of ranked alignment methods.
  We report improvement over the corresponding unaligned base model
  in Spearman correlation (left) and Top-1 accuracy (right).
  Error bars denote $\pm$ one standard error over five training seeds.}
  \label{fig:cross-family}
\end{figure*}
This behavior is consistent with the role of $\beta$ in the standard
KL-regularized alignment objective, where $\beta$ controls the
strength of regularization toward $\pi_{\mathrm{ref}}$. As
$\beta\to0$, the exponential reweighting induced by the alignment
signal increasingly dominates the finite reference-policy odds.
Consequently, subject to positive reference support, the ordering of
the optimal policy is governed primarily by the learned ranked
preference signal rather than by the reference policy.

\subsection{Reduction to Pairwise Preference Optimization}
\label{sec:binary-reduction}

When $k=2$, there is only one rank comparison, with
$\rho_x(1)=q_x(2\mid1)=1$ trivially. Writing
$y^+=y_{(1)}$ and $y^-=y_{(2)}$, the aligned and unaligned
distributions reduce to the preferred and rejected response
distributions, respectively. Equation~\eqref{eq:frpo-loss} becomes
\begin{equation}
\begin{aligned}
\mathcal L_{f\text{-RPO}}(\theta)
={}&-
\mathbb E_{(x,y^+)\sim\mathcal D^+}
\!\left[g_f(r_\theta(x,y^+))\right]
\\
&+
\mathbb E_{(x,y^-)\sim\mathcal D^-}
\!\left[
f^*\!\left(g_f(r_\theta(x,y^-))\right)
\right].
\end{aligned}
\label{eq:frpo-binary-reduction}
\end{equation}
This is exactly the binary divergence-based alignment objective
Eq.~\eqref{eq:background-divergence-alignment} proposed by \cite{haldar2025llm}. Hence, pairwise
preference alignment is recovered as the $k=2$ special case of
$f$-RPO.

\section{Experiments}
\label{sec: experiments}
We evaluate $f$-RPO against two strong ranked-alignment baselines:
PL-DPO, the Plackett--Luce extension of DPO
\citep{song2024preference}, and Mallows-RCPO
\citep{tang2026beyond}, which combines the implicit DPO reward with a
distance-based ranked-choice model. We additionally compare ranked
alignment against standard pairwise DPO in
\sref{sec:pvsr} to isolate the benefit of learning directly
from ordered feedback.

\paragraph{Dataset and training.}
We align non-instruction-tuned base models from the Qwen, Llama, and
Gemma families, allowing differences after training to be attributed
more directly to the alignment objective.
Following the ranked-feedback setting of
\citet{tang2026beyond}, we construct five-way preference data from
\texttt{UltraFeedback} prompts
\citep{cui2023ultrafeedback,pmlr-v235-cui24f}.
For each prompt, five instruction-tuned generators of varying
capabilities produce candidate responses, which are scored by
\texttt{Skywork-Reward-V2-Llama-3.1-8B}
\citep{liu2024skywork} and sorted to obtain the target ranking.
We retain only high-confidence rankings with sufficiently separated
adjacent reward scores. Complete dataset construction, filtering, model
configurations, and training hyperparameters are provided in
Appendix~\sref{sec: ranked_data} and
Appendix~\sref{sec: training_details}. Refer to the code repository \href{https://github.com/rhaldarpurdue/f-RPO}{[link]} for implementation details.

\paragraph{Evaluation.}
We evaluate both \emph{ranking quality} and \emph{generation quality}.
For ranking evaluation, each fixed candidate response is assigned its
length-normalized log-likelihood under the aligned policy, and the
resulting ordering is compared with the target ranking.
Our primary metrics are Spearman's rank correlation $\rho$, which
measures agreement over the complete ranking, and Top-1 accuracy,
which measures recovery of the highest-ranked response.
Additional rank metrics like Kendall's $\tau$, pairwise accuracy, exact
ranking match, rank MAE, and rank RMSE are provided in the Appendix; these metrics show trends consistent with the primary measures and are therefore omitted from the main text. Full definitions are provided in
Appendix~\ref{app:ranking-methodology}.

For generation quality, we use AlpacaEval~2.0
\citep{alpaca}. Each aligned checkpoint is evaluated on the same
805 instructions and compared against generations from its
corresponding unaligned base model, with
\texttt{gpt-4.1-mini-2025-04-14} serving as the pairwise judge.
We report the resulting base-relative win rate.
Further details are provided in
Appendix~\ref{app:alpacaeval-methodology}.

\subsection{Main Results}
Figure~\ref{fig:cross-family} summarizes ranking performance across
model families and scales. All ranked-alignment methods improve
Spearman correlation relative to their corresponding unaligned
baselines, demonstrating that preference optimization improves the
model's ability to recover the target response ordering.
Across models, $f$-RPO consistently yields the largest or among the
largest gains in Spearman correlation, while remaining competitive in
Top-1 retrieval. The advantage is particularly pronounced for several
of the smaller base models, where there is greater room for alignment
to improve the initial ranking behavior.
The expanded evaluation in the appendix (Tables \ref{tab:scaling-k5-test-qwen2515b}, \ref{tab:scaling-k5-test-qwen257b}) shows the same overall trend
across Kendall correlation, pairwise accuracy, exact match, and rank
error metrics.
\begin{table}[!ht]
\centering
\caption{
AlpacaEval win rate (\%) against the corresponding base model.
Best and second-best results in each column are
\textbf{bold} and \underline{underlined}, respectively.
}
\label{tab:alpacaeval-k5}
\small
\setlength{\tabcolsep}{3.5pt}

\resizebox{\columnwidth}{!}{%
\begin{tabular}{lrrrr}
\toprule
Method
& Qwen 1.5B
& Llama 8B
& Gemma 2B
& Llama 3B \\
\midrule

\rowcolor{gray!15}
$f$-RPO-RKL
& \underline{70.10}
& \underline{51.18}
& \textbf{69.13}
& \underline{88.26} \\

\rowcolor{gray!15}
$f$-RPO-TV
& \textbf{81.24}
& \textbf{55.27}
& \underline{65.84}
& \textbf{94.84} \\

Mallows-RCPO
& 58.80
& 50.80
& 61.73
& 83.16 \\

PL-DPO
& 56.50
& 48.57
& 51.30
& 68.91 \\

\bottomrule
\end{tabular}%
}
\end{table}

Table~\ref{tab:alpacaeval-k5} evaluates whether these improvements in
ranking quality translate to generation quality.
The total-variation and reverse-KL variants of $f$-RPO obtain the
strongest win rates across the evaluated model families, with
$f$-RPO-TV achieving the best result on three of the four models.
Together with Figure~\ref{fig:cross-family}, these results indicate
that $f$-RPO improves both recovery of ranked preferences and
downstream response quality.

Representative generations illustrating qualitative differences
between the alignment objectives are provided in
Appendix~\ref{app:qualitative}.
\subsection{Ablations}

We next examine how ranking depth and the choice of $f$-divergence
affect $f$-RPO.

\paragraph{Effect of ranking depth.}
Figure~\ref{fig:ranking-depth} varies the number of ranked candidates
$k$ used during alignment. Increasing the amount of ranked feedback
improves performance on both evaluation metrics, with the largest
relative gains appearing in Spearman correlation. Top-1 accuracy also
increases as $k$ grows, indicating that richer ranked supervision
benefits both recovery of the complete ordering and identification of
the highest-ranked response.

\begin{figure}[!ht]
  \centering
  \includegraphics[
      width=0.8\columnwidth,
      trim=0.8cm 0 0 0.5cm,
      clip
  ]{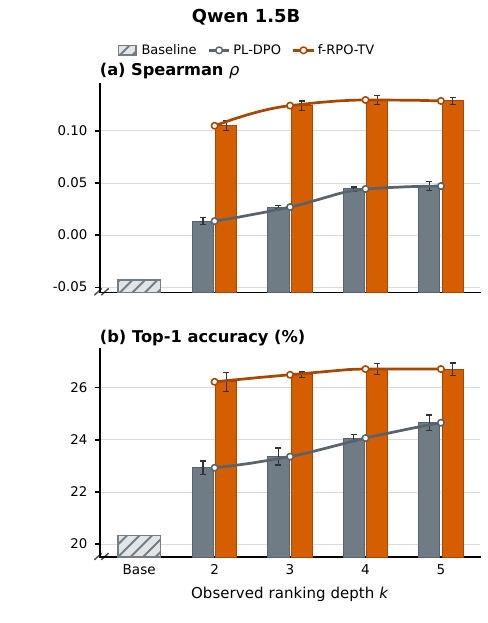}
  \caption{
  Effect of ranking depth $k$ on Qwen2.5-1.5B.
  Bars show ranking performance for PL-DPO and $f$-RPO-TV, while the
  overlaid curves highlight the trend across increasing $k$.
  }
  \label{fig:ranking-depth}
\end{figure}

\paragraph{Choice of $f$-divergence.}
\label{sec: fdiv ablation}
The theory in Section \sref{sec: theory} characterizes the
behavior of the $f$-RPO family at population level optimality, but finite-data optimization can depend substantially on the choice of divergence.
Table~\ref{tab:frpo-divergence-main} compares the divergences used in
our framework. All variants improve substantially over the unaligned
base model, indicating that the divergence-based ranked objective
learns meaningful preference structure. Reverse KL and total variation
provide the strongest overall ranking performance, motivating their
use as the principal $f$-RPO variants throughout the main experiments.
Results over the complete set of ranking metrics are reported in
Table~\ref{tab:frpo-divergence-k5-all-test}.

\begin{table}[!ht]
\centering
\small
\caption{
Ablation on the choice of $f$-divergence on ranking performance.
Mean (SE) reported over five seeds.
Best and second-best aligned results are \textbf{bold} and
\underline{underlined}, respectively. Model: Qwen2.5-0.5B.
}
\label{tab:frpo-divergence-main}

\resizebox{\columnwidth}{!}{%
\begin{tabular}{lccc}
\toprule
Divergence
& Spearman $\rho$
& Pairwise (\%)
& Top-1 (\%) \\
\midrule

Reverse KL
& \textbf{0.1648 (0.0039)}
& \textbf{56.70 (0.16)}
& \textbf{28.94 (0.10)} \\

Total variation
& \underline{0.1492 (0.0017)}
& \underline{56.16 (0.08)}
& \underline{28.08 (0.07)} \\

KL
& 0.0928 (0.0019)
& 53.98 (0.08)
& 26.94 (0.29) \\

Pearson $\chi^2$
& 0.1048 (0.0012)
& 54.40 (0.05)
& 27.70 (0.26) \\

Squared Hellinger
& 0.0937 (0.0005)
& 53.96 (0.03)
& 27.80 (0.18) \\

Jensen--Shannon
& 0.1044 (0.0010)
& 54.41 (0.05)
& 27.97 (0.16) \\

\midrule
Base
& 0.0314
& 51.60
& 23.85 \\

\bottomrule
\end{tabular}%
}
\end{table}

\subsection{Pairwise vs.\ Ranked Alignment}
\label{sec:pvsr}

Finally, we isolate the benefit of training directly from ranked
feedback. A ranking over $k$ responses can be decomposed into
$\binom{k}{2}$ binary comparisons and optimized using a pairwise
objective such as DPO; however, this discards the listwise structure
of the original observation and expands each ranked example into
multiple training comparisons.

Table~\ref{tab:pairwise-vs-ranked} compares standard DPO with its
ranked Plackett--Luce counterpart, PL-DPO, under the same
optimization-step budget. PL-DPO consistently improves Spearman
correlation over pairwise DPO across model scales, demonstrating
an empirical benefit from exploiting the ranked observation directly.
The $f$-RPO variants further increase this advantage, with the largest
gains observed under reverse KL and total variation.



\begin{table}[!ht]
\centering
\caption{
Pairwise versus ranked preference alignment under the same
optimization-step budget. PL-DPO is the ranked counterpart of DPO.
Entries report mean Spearman correlation $\rho$ (SE) over five seeds.
}
\label{tab:pairwise-vs-ranked}

\small
\setlength{\tabcolsep}{3.2pt}

\resizebox{\columnwidth}{!}{%
\begin{tabular}{lccc}
\toprule
Method
& Qwen2.5 0.5B
& Qwen2.5 1.5B
& Qwen2.5 7B \\
\midrule

Base
& 0.0314
& -0.0431
& 0.4390 \\

DPO (pairwise)
& 0.0996 (0.0014)
& 0.0420 (0.0041)
& 0.4480 (0.0005) \\

PL-DPO (ranked)
& 0.1131 (0.0017)
& 0.0471 (0.0045)
& 0.4530 (0.0006) \\

\rowcolor{gray!10}
$f$-RPO-TV (ranked)
& \underline{0.1492 (0.0017)}
& \underline{0.1288 (0.0035)}
& \textbf{0.4858 (0.0005)} \\

\rowcolor{gray!10}
$f$-RPO-RKL (ranked)
& \textbf{0.1648 (0.0039)}
& \textbf{0.1538 (0.0068)}
& \underline{0.4827 (0.0008)} \\

\midrule

$\Delta$ PL-DPO $-$ DPO
& +0.0135
& +0.0051
& +0.0050 \\

$\Delta$ best $f$-RPO $-$ DPO
& \textbf{+0.0652}
& \textbf{+0.1118}
& \textbf{+0.0378} \\

\bottomrule
\end{tabular}%
}
\end{table}
\section{Conclusion}

We introduced $f$-RPO, a divergence-based framework for ranked
preference alignment. By extending the notions of aligned and
unaligned response distributions to ranked feedback, the framework
yields a family of objectives indexed by the choice of
$f$-divergence. We showed that, at the population optimum, the learned
policy is reweighted according to the aligned-to-unaligned density
ratio and preserves the underlying
preference ordering. When the ranking contains only two responses,
$f$-RPO reduces to the binary divergence-based direct alignment
framework of \citet{haldar2025llm}.

Empirically, all evaluated $f$-RPO variants improve ranking quality
relative to the unaligned base models, with total variation and
reverse KL providing the strongest overall performance. Across model
families and scales, these variants outperform strong ranked-alignment
baselines in both ranking recovery and generation quality. Our
pairwise-versus-ranked comparison further shows that directly
optimizing from ranked feedback can provide stronger ranking alignment
than reducing the same supervision to binary comparisons. Overall,
our results establish ranked divergence optimization as a principled
and effective extension of preference alignment beyond pairwise
feedback.

\bibliographystyle{apalike}
\bibliography{references}

\clearpage
\appendix
\thispagestyle{empty}

\onecolumn
\appendix
\section{Additional Background and Definitions}
\subsection{Ranked Response Distribution}
\label{app:rank-conditional-density}

We expand on the ranked response density $\mathcal D_x^i$,
introduced in Eq.~\eqref{eq:rank-conditional-density}.

Fix a prompt $x$ and consider $k$ independent candidate responses
\[
    y_1,\ldots,y_k
    \overset{\mathrm{i.i.d.}}{\sim}
    \mathcal D_x,
\]
where $\mathcal D_x$ has density $p_x(y)$.
Let $u_x(y)$ denote the underlying utility with corresponding CDF
\[
    F_x(t)
    =
    \Pr_{y\sim\mathcal D_x}
    \bigl(u_x(y)\le t\bigr).
\]
Assume that the utility distribution is continuous, so that ties
occur with probability zero.

The ranked responses $y_{(1)},\ldots,y_{(k)}$ are obtained by sorting
the candidates in decreasing order of utility.
Let $\mathcal D_x^i$ denote the distribution of $y_{(i)}$.

\begin{lemma}[Rank-conditional density]
\label{lem:rank-conditional-density}
The density of the response occupying rank $i$ is
\begin{equation}
    p_x^i(y)
    =
    \frac{k!}{(i-1)!(k-i)!}\,
    p_x(y)\,
    F_x(u_x(y))^{k-i}
    \bigl[1-F_x(u_x(y))\bigr]^{i-1}.
\end{equation}
\end{lemma}

\begin{proof}
Fix a candidate response $y$ with utility $u_x(y)$.
For this response to occupy rank $i$, exactly $i-1$ of the
remaining $k-1$ candidates must have greater utility, while the
other $k-i$ candidates must have lower utility.

By independence, the probability of a particular assignment of
these higher- and lower-utility candidates is
\[
    \bigl[1-F_x(u_x(y))\bigr]^{i-1}
    F_x(u_x(y))^{k-i}.
\]
There are $\binom{k-1}{i-1}$ ways to choose which candidates
have greater utility. Moreover, any of the $k$ exchangeable
candidates may occupy rank $i$.

Multiplying by the base density $p_x(y)$, we obtain
\begin{align}
    p_x^i(y)
    &=
    k\binom{k-1}{i-1}
    p_x(y)
    F_x(u_x(y))^{k-i}
    \bigl[1-F_x(u_x(y))\bigr]^{i-1}
    \nonumber\\
    &=
    \frac{k!}{(i-1)!(k-i)!}\,
    p_x(y)\,
    F_x(u_x(y))^{k-i}
    \bigl[1-F_x(u_x(y))\bigr]^{i-1}.
\end{align}
\end{proof}

\subsection{$f$-Divergences and Variational Representation}
\label{app:f-divergence-background}

Let $P$ and $Q$ be probability measures on a measurable space
$\mathcal Z$, with $P\ll Q$. Let
\[
    R(z)
    :=
    \frac{dP}{dQ}(z)
\]
denote their density ratio.

\begin{definition}[$f$-divergence]
Let $f:\mathbb R_+\to\mathbb R$ be convex with $f(1)=0$.
The $f$-divergence between $P$ and $Q$ is
\begin{equation}
    D_f(P\|Q)
    :=
    \mathbb E_{z\sim Q}
    \left[
        f\!\left(R(z)\right)
    \right].
    \label{eq:app-f-divergence}
\end{equation}
\end{definition}

The convex conjugate of $f$ is defined as
\begin{equation}
    f^*(t)
    :=
    \sup_{u\ge 0}
    \left\{
        tu-f(u)
    \right\}.
    \label{eq:app-convex-conjugate}
\end{equation}

The following standard variational representation forms the basis of
our divergence-based alignment objective.

\begin{lemma}[Variational representation of $f$-divergence]
\label{lem:app-f-variational}
For probability measures $P\ll Q$,
\begin{equation}
    D_f(P\|Q)
    =
    \sup_T
    \left\{
        \mathbb E_{z\sim P}[T(z)]
        -
        \mathbb E_{z\sim Q}[f^*(T(z))]
    \right\},
    \label{eq:app-f-variational}
\end{equation}
where the supremum is over measurable functions
$T:\mathcal Z\to\operatorname{dom}(f^*)$.
If $f$ is differentiable, an optimal functional is
\begin{equation}
    T^*(z)
    =
    f'\!\left(
        \frac{dP}{dQ}(z)
    \right).
    \label{eq:app-optimal-critic}
\end{equation}
More generally,
$T^*(z)\in\partial f(R(z))$.
\end{lemma}

\begin{proof}
By the Fenchel--Young inequality,
\[
    ut
    \le
    f(u)+f^*(t),
\]
with equality if and only if
$t\in\partial f(u)$.
Taking
$u=R(z)=dP/dQ(z)$ and $t=T(z)$ gives
\[
    R(z)T(z)-f^*(T(z))
    \le
    f(R(z)).
\]
Taking expectations under $Q$ and using
$dP=R\,dQ$ yields
\begin{align}
    \mathbb E_{P}[T]
    -
    \mathbb E_{Q}[f^*(T)]
    &=
    \mathbb E_{Q}
    \left[
        R\,T-f^*(T)
    \right]
    \\
    &\le
    \mathbb E_{Q}[f(R)]
    =
    D_f(P\|Q).
\end{align}
Equality is attained pointwise whenever
\[
    T^*(z)\in\partial f(R(z)).
\]
If $f$ is differentiable, this reduces to
$T^*(z)=f'(R(z))$, proving the result.
\end{proof}
\subsection{Canonical Link Functions}
\label{app:canonical-links}

The variational critic $T(z)$ in
Eq.~\eqref{eq:app-f-variational} takes values in the effective domain
$\operatorname{dom}(f^*)$ of the convex conjugate. In contrast, the
implicit reward $r_\theta(x,y)$ is real-valued. To parameterize the
critic using the implicit reward while respecting the domain of
$f^*$, we use a monotone link
\[
    g_f:\mathbb R\rightarrow \operatorname{dom}(f^*),
    \qquad
    T_\theta(x,y)=g_f(r_\theta(x,y)).
\]
The monotonicity of $g_f$ is also important for alignment: ordering
the critic values is equivalent to ordering the corresponding
implicit rewards. Following \citet{haldar2026f}, we use the canonical
links summarized in Table~\ref{tab:canonical-links}.

\begin{table}[t]
\centering
\caption{Common $f$-divergences, convex conjugates, effective domains,
and canonical link functions used in our objectives.}
\label{tab:canonical-links}
\resizebox{0.8\linewidth}{!}{
\begin{tabular}{lcccc}
\toprule
Divergence
& $f(t)$
& $f^*(u)$
& $\operatorname{dom}(f^*)$
& $g_f(r)$ \\
\midrule

KL
& $t\log t$
& $e^{u-1}$
& $\mathbb R$
& $r$
\\[1mm]

Reverse KL
& $-\log t$
& $-1-\log(-u)$
& $u<0$
& $-e^{-r}$
\\[1mm]

Pearson $\chi^2$
& $(t-1)^2$
& $\frac14u^2+u$
& $\mathbb R$
& $r$
\\[1mm]

Hellinger
& $(\sqrt{t}-1)^2$
& $\dfrac{u}{1-u}$
& $u<1$
& $1-e^{-r}$
\\[2mm]

Jensen--Shannon
& $t\log t-(t+1)\log\frac{t+1}{2}$
& $-\log(2-e^u)$
& $u<\log 2$
& $\log 2-\log(1+e^{-r})$
\\[2mm]

Total variation
& $\frac12|t-1|$
& $u$
& $|u|\le \frac12$
& $\frac12(1+e^{-r})^{-1}$
\\

\bottomrule
\end{tabular}
}
\end{table}

Substituting the linked critic into the variational objective gives
\begin{equation}
    \mathcal J_f(\theta)
    =
    \mathbb E_{P^+}
    \!\left[g_f(r_\theta)\right]
    -
    \mathbb E_{P^-}
    \!\left[
        f^*\!\left(g_f(r_\theta)\right)
    \right],
    \label{eq:app-linked-f-objective}
\end{equation}
which provides the basis for divergence-based alignment objectives.
\subsection{Divergence-Based Alignment Objectives}
\label{app:divergence-alignment}

The variational objective in
Eq.~\eqref{eq:app-linked-f-objective} provides a general template for
alignment once an \emph{aligned} distribution $\mathcal D^+$ and an
\emph{unaligned} distribution $\mathcal D^-$ are specified:
\begin{equation}
    \mathcal L_f(\theta)
    =
    -\mathbb E_{(x,y)\sim\mathcal D^+}
    \!\left[g_f(r_\theta(x,y))\right]
    +
    \mathbb E_{(x,y)\sim\mathcal D^-}
    \!\left[
        f^*\!\left(g_f(r_\theta(x,y))\right)
    \right].
    \label{eq:generic-divergence-alignment}
\end{equation}
Different choices of $\mathcal D^+$ and $\mathcal D^-$ recover
different alignment settings. We briefly review existing examples
below and summarize them in Table~\ref{tab:divergence-alignment-summary}.

\paragraph{Binary preference alignment.}
For binary preference data $(x,y^+,y^-)$, the aligned and unaligned
distributions are induced by the preferred and rejected responses,
respectively. The resulting $f$-divergence direct optimization (FDO)
objective is
\begin{equation}
\begin{aligned}
    \mathcal L_{\mathrm{FDO}}^{(f,g)}(\theta)
    ={}&
    -\mathbb E_{(x,y)\sim\mathcal D^+}
    \!\left[g_f(r_\theta(x,y))\right]
    \\
    &+
    \mathbb E_{(x,y)\sim\mathcal D^-}
    \!\left[
        f^*\!\left(g_f(r_\theta(x,y))\right)
    \right],
\end{aligned}
\label{eq:app-fdo}
\end{equation}
where $\mathcal D^+$ and $\mathcal D^-$ correspond to the chosen and
rejected response distributions. Thus, binary preference optimization
can be viewed directly as variational divergence estimation between
preferred and less-preferred responses. \citep{haldar2025llm}

\paragraph{Reward-based alignment.}
The same view extends to on-policy reward optimization through
$f$-GRPO. Given a prompt $x$, responses
$y_1,\ldots,y_G\sim\pi_{\mathrm{old}}(\cdot\mid x)$ are sampled and
assigned normalized advantages $a_1,\ldots,a_G$, where $a_i
    =
    \frac{
        r(x,y_i)-\frac{1}{G}\sum_{j=1}^G r(x,y_j)
    }{
        \operatorname{Std}\{r(x,y_j)\}_{j=1}^G+\epsilon
    }$. Responses with
positive advantage contribute to the aligned term, while responses
with non-positive advantage contribute to the unaligned term.
Using
\[
    r_{\theta,i}
    =
    \beta
    \log
    \frac{\pi_\theta(y_i\mid x)}
         {\pi_{\mathrm{ref}}(y_i\mid x)},
\]
the corresponding functional is
\begin{equation}
\psi_f(r_{\theta,i},a_i)
=
\begin{cases}
\vec{\sigma}_i
\!\left(
\{r_j-\log\pi_{\mathrm{old}}(y_j\mid x)\}_{j=1}^G
\right)
g_f(r_{\theta,i}),
& a_i>0,
\\[4pt]
\vec{\sigma}_i
\!\left(
\{-r_j-\log\pi_{\mathrm{old}}(y_j\mid x)\}_{j=1}^G
\right)
f^*\!\left(g_f(r_{\theta,i})\right),
& a_i\le 0,
\end{cases}
\label{eq:app-fgrpo-functional}
\end{equation}
where $\vec{\sigma}_i(\cdot)$ denotes the $i$-th softmax weight over
the sampled responses. The $f$-GRPO loss is
\begin{equation}
    \mathcal L_{f\text{-GRPO}}^{(f,g)}(\theta)
    =
    \mathbb E_x
    \left[
        -\frac{1+\beta^{-1}}{G}
        \sum_{i=1}^G
        a_i\,
        \psi_f(r_{\theta,i},a_i)
    \right].
    \label{eq:app-fgrpo}
\end{equation}
At the population level, this induces an aligned distribution over
above-average-reward responses and an unaligned distribution over
below-average-reward responses, with the softmax weights providing their normalized importance within each rollout group. \citep{haldar2026f}

\paragraph{Hybrid alignment.}
The $f$-HAL objective combines off-policy binary preference alignment
with on-policy reward-based alignment:
\begin{equation}
    \mathcal L_{f\text{-HAL}}^{(f,g)}(\theta)
    =
    \lambda
    \mathcal L_{\mathrm{FDO}}^{(f,g)}(\theta)
    +
    (1-\lambda)
    \mathcal L_{f\text{-GRPO}}^{(f,g)}(\theta),
    \label{eq:app-fhal}
\end{equation}
where $\lambda\in[0,1]$ controls the relative contribution of
preference supervision and reward-based optimization. \citep{haldar2026f}

These objectives illustrate that the variational $f$-divergence
framework is determined primarily by how the aligned and unaligned
response distributions are constructed. In this work, we extend the
same principle to ranked preference feedback.
\begin{table}[t]
\centering
\caption{
Existing divergence-based alignment objectives differ primarily in
how the aligned and unaligned response distributions are defined.
}
\label{tab:divergence-alignment-summary}
\resizebox{\columnwidth}{!}{
\begin{tabular}{lccc}
\toprule
Method
& Feedback
& Aligned $\mathcal D^+$
& Unaligned $\mathcal D^-$
\\
\midrule

FDO
& Pairwise preference 
& Chosen responses
& Rejected responses
\\

$f$-GRPO
& Scalar reward
& Positive-advantage responses
& Non-positive-advantage responses
\\

$f$-HAL
& Pairwise Preference $+$ reward
& Chosen $+$ positive-advantage
& Rejected $+$ non-positive-advantage
\\

$f$-RPO (ours)
& Ranked preference
& Higher-ranked responses
& Lower-ranked responses
\\

\bottomrule
\end{tabular}
}
\end{table}
\subsection{Derivation of the Ranked $f$-Divergence Objective}
\label{app:frpo-derivation}

Recall the ranked aligned and unaligned distributions defined in
Eqs.~\eqref{eq:rank-aligned-dist} and
\eqref{eq:rank-unaligned-dist}:
\begin{align}
    d\mathcal D^+(x,y)
    &=
    p_{\mathcal X}(x)
    \sum_{i=1}^{k-1}
    \rho_x(i)p_x^i(y),
    \label{eq:app-ranked-positive}
    \\
    d\mathcal D^-(x,y)
    &=
    p_{\mathcal X}(x)
    \sum_{i=1}^{k-1}
    \rho_x(i)
    \sum_{j=i+1}^{k}
    q_x(j\mid i)p_x^j(y).
    \label{eq:app-ranked-negative}
\end{align}

Applying the variational representation of the $f$-divergence gives
\begin{equation}
\begin{aligned}
D_f(\mathcal D^+\|\mathcal D^-)
=
\sup_T
\Big\{
    \mathbb E_{\mathcal D^+}[T(x,y)]
    -
    \mathbb E_{\mathcal D^-}[f^*(T(x,y))]
\Big\}.
\end{aligned}
\label{eq:frpo-variational-appendix}
\end{equation}

Using the definitions above, the aligned expectation can be written as
\begin{align}
    \mathbb E_{\mathcal D^+}[T]
    =
    \mathbb E_{x,\mathbf y_{(1:k)}}
    \sum_{i=1}^{k-1}
    \rho_x(i)\,
    T(x,y_{(i)}),
    \label{eq:frpo-positive-expectation}
\end{align}
while the unaligned expectation becomes
\begin{align}
    \mathbb E_{\mathcal D^-}[f^*(T)]
    =
    \mathbb E_{x,\mathbf y_{(1:k)}}
    \sum_{i=1}^{k-1}
    \rho_x(i)
    \sum_{j=i+1}^{k}
    q_x(j\mid i)\,
    f^*(T(x,y_{(j)})).
    \label{eq:frpo-negative-expectation}
\end{align}

Therefore,
\begin{align}
D_f(\mathcal D^+\|\mathcal D^-)
=
\sup_T\;
\mathbb E_{x,\mathbf y_{(1:k)}}
\sum_{i=1}^{k-1}
\rho_x(i)
\Bigg[
    T(x,y_{(i)})
    -
    \sum_{j=i+1}^{k}
    q_x(j\mid i)
    f^*(T(x,y_{(j)}))
\Bigg].
\label{eq:frpo-ranked-variational}
\end{align}

We parameterize the variational critic using the implicit policy reward,
\begin{equation}
    T_\theta(x,y)
    =
    g_f(r_\theta(x,y)),
    \qquad
    r_\theta(x,y)
    =
    \beta
    \log
    \frac{\pi_\theta(y\mid x)}
         {\pi_{\mathrm{ref}}(y\mid x)}.
\end{equation}
Substitution into Eq.~\eqref{eq:frpo-ranked-variational} yields
\begin{align}
    \mathcal J_{f\text{-RPO}}(\theta)
    =
    \mathbb E_{x,\mathbf y_{(1:k)}}
    \sum_{i=1}^{k-1}\rho_x(i)
    \Bigg[
        g_f(r_{\theta,i})
        -
        \sum_{j=i+1}^{k}
        q_x(j\mid i)
        f^*\!\left(g_f(r_{\theta,j})\right)
    \Bigg],
    \label{eq:frpo-objective-appendix}
\end{align}
where $r_{\theta,i}=r_\theta(x,y_{(i)})$.
The training objective is
\begin{equation}
    \mathcal L_{f\text{-RPO}}(\theta)
    =
    -\mathcal J_{f\text{-RPO}}(\theta),
\end{equation}
which recovers Eq.~\eqref{eq:frpo-loss}.
\begin{table*}[t]
\centering
\small
\setlength{\tabcolsep}{4pt}
\caption{
$f$-divergence instantiations considered in this work.
The final column gives the population scaling
$g_f^{-1}(f'(t))$, where
$t=d\mathcal D^+/d\mathcal D^-$.
Terms independent of $\theta$ are omitted from the displayed objectives.
}
\label{tab:frpo-all-objectives}

\resizebox{\textwidth}{!}{%
\begin{tabular}{lcccc}
\toprule
Divergence
& $f(t)$
& $g_f(r)$
& $\mathcal L_{f\text{-RPO}}(\theta)$
& $g_f^{-1}(f'(t))$
\\
\midrule

Reverse KL
&
$-\log t$
&
$-e^{-r}$
&
$\displaystyle
\mathbb E\sum_i\rho_x(i)
\left[
e^{-r_{\theta,i}}
+
\sum_{j>i}q_x(j\mid i)r_{\theta,j}
\right]
$
&
$\log t$
\\[3mm]

Total variation
&
$\frac12|t-1|$
&
$\frac12\sigma(r)$
&
$\displaystyle
\frac12\mathbb E\sum_i\rho_x(i)
\left[
-\sigma(r_{\theta,i})
+
\sum_{j>i}q_x(j\mid i)\sigma(r_{\theta,j})
\right]
$
&
$\displaystyle
-\log\!\left(
(\operatorname{sign}(t-1))^{-1}-1
\right)
$
\\[3mm]

KL
&
$t\log t$
&
$r$
&
$\displaystyle
\mathbb E\sum_i\rho_x(i)
\left[
-r_{\theta,i}
+
\sum_{j>i}q_x(j\mid i)e^{r_{\theta,j}-1}
\right]
$
&
$\log t+1$
\\[3mm]

Pearson $\chi^2$
&
$(t-1)^2$
&
$r$
&
$\displaystyle
\mathbb E\sum_i\rho_x(i)
\left[
-r_{\theta,i}
+
\sum_{j>i}q_x(j\mid i)
\left(
\frac14 r_{\theta,j}^2+r_{\theta,j}
\right)
\right]
$
&
$2(t-1)$
\\[3mm]

Squared Hellinger
&
$(\sqrt t-1)^2$
&
$1-e^{-r}$
&
$\displaystyle
\mathbb E\sum_i\rho_x(i)
\left[
e^{-r_{\theta,i}}
+
\sum_{j>i}q_x(j\mid i)e^{r_{\theta,j}}
\right]
$
&
$\frac12\log t$
\\[3mm]

Jensen--Shannon
&
$\displaystyle
t\log t-(t+1)\log\frac{t+1}{2}
$
&
$\displaystyle
\log 2-\log(1+e^{-r})
$
&
$\displaystyle
\mathbb E\sum_i\rho_x(i)
\left[
-\log\!\bigl(2\sigma(r_{\theta,i})\bigr)
-
\sum_{j>i}q_x(j\mid i)
\log\!\bigl(2\sigma(-r_{\theta,j})\bigr)
\right]
$
&
$\log t$
\\

\bottomrule
\end{tabular}%
}
\end{table*}
\section{Proofs}
\label{sec: proofs}
\subsection{Proof of Lemma~\ref{lem:density_ratio_util}}
\label{app:proof-density-ratio-utility}

\begin{proof}
Fix a prompt $x$ and suppress its dependence in the notation.
Using the rank masses from Eq.~\eqref{eq:rank-masses}, the conditional
aligned and unaligned densities can be written as
\begin{equation}
\begin{aligned}
    d\mathcal D^+(y\mid x)
    &=
    \sum_{i=1}^{k}
    B_i\,p_x^i(y),
    \\
    d\mathcal D^-(y\mid x)
    &=
    \sum_{i=1}^{k}
    A_i\,p_x^i(y).
\end{aligned}
\label{eq:app-rank-mass-mixtures}
\end{equation}

Let $y,y'$ satisfy $u_x(y)>u_x(y')$. From
Eq.~\eqref{eq:rank-conditional-density}, for $i<j$,
\begin{align}
\frac{p_x^i(y)}{p_x^j(y)}
&\propto
\left(
    \frac{F_x(u_x(y))}
         {1-F_x(u_x(y))}
\right)^{j-i}.
\end{align}
Since $F_x$ is nondecreasing,
\begin{equation}
    p_x^i(y)p_x^j(y')
    \ge
    p_x^i(y')p_x^j(y),
    \qquad i<j.
    \label{eq:rank-density-order}
\end{equation}

Now consider the cross-product difference
\begin{align}
&
d\mathcal D^+(y\mid x)
d\mathcal D^-(y'\mid x)
-
d\mathcal D^+(y'\mid x)
d\mathcal D^-(y\mid x)
\nonumber\\
&=
\sum_{i<j}
(B_iA_j-B_jA_i)
\Big[
    p_x^i(y)p_x^j(y')
    -
    p_x^i(y')p_x^j(y)
\Big].
\label{eq:app-density-ratio-cross-product}
\end{align}
By the rank-favored condition
$B_iA_j\ge B_jA_i$ for $i<j$, and by
Eq.~\eqref{eq:rank-density-order}, every summand is nonnegative.
Therefore,
\[
d\mathcal D^+(y\mid x)
d\mathcal D^-(y'\mid x)
\ge
d\mathcal D^+(y'\mid x)
d\mathcal D^-(y\mid x).
\]
Whenever the unaligned densities are positive, dividing by
$d\mathcal D^-(y\mid x)d\mathcal D^-(y'\mid x)$ yields
\[
    R_x(y)\ge R_x(y').
\]
Hence the aligned-to-unaligned density ratio is nondecreasing in the
underlying utility.
\end{proof}

\subsection{Proof of Theorem~\ref{thm:divergence-alignment}}
\label{app:proof-divergence-alignment}

\begin{proof}
By Eq.~\eqref{eq:frpo-variational-appendix},
\begin{align}
D_f(\mathcal D^+\|\mathcal D^-)
=
\sup_T
\left\{
    \mathbb E_{\mathcal D^+}[T(x,y)]
    -
    \mathbb E_{\mathcal D^-}[f^*(T(x,y))]
\right\}.
\label{eq:app-thm1-variational}
\end{align}
The $f$-RPO objective parameterizes the critic as
\[
    T_\theta(x,y)
    =
    g_f(r_\theta(x,y)).
\]
Therefore,
\[
    -\mathcal L_{f\text{-RPO}}(\theta)
    =
    \mathbb E_{\mathcal D^+}[T_\theta]
    -
    \mathbb E_{\mathcal D^-}[f^*(T_\theta)].
\]

Under exact realizability of the unrestricted variational optimum by
the policy-parameterized critic, there exists $\theta^*$ such that
$T_{\theta^*}=T^*$. Hence,
\begin{equation}
    -\mathcal L_{f\text{-RPO}}(\theta^*)
    =
    D_f(\mathcal D^+\|\mathcal D^-).
\end{equation}

From Lemma~\ref{lem:app-f-variational}, an optimal variational critic
satisfies
\begin{equation}
    T^*(x,y)
    \in
    \partial f(R_x(y)),
    \qquad
    R_x(y)
    =
    \frac{d\mathcal D^+}{d\mathcal D^-}(x,y).
    \label{eq:app-optimal-ranked-critic}
\end{equation}
For differentiable $f$, this becomes
\begin{equation}
    T^*(x,y)
    =
    f'(R_x(y)).
\end{equation}

Under critic realizability,
\begin{equation}
    g_f(r_{\theta^*}(x,y))
    =
    f'(R_x(y)).
\end{equation}
Applying $g_f^{-1}$ gives
\begin{equation}
    r_{\theta^*}(x,y)
    =
    g_f^{-1}\!\left(f'(R_x(y))\right).
    \label{eq:app-optimal-implicit-reward}
\end{equation}

Finally, using the implicit-reward parameterization
\[
    r_{\theta^*}(x,y)
    =
    \beta
    \log
    \frac{\pi_{\theta^*}(y\mid x)}
         {\pi_{\mathrm{ref}}(y\mid x)},
\]
we obtain
\[
    \frac{\pi_{\theta^*}(y\mid x)}
         {\pi_{\mathrm{ref}}(y\mid x)}
    =
    \exp\!\left(
        \frac{
            g_f^{-1}(f'(R_x(y)))
        }{\beta}
    \right).
\]
Equivalently,
\begin{equation}
    \pi_{\theta^*}(y\mid x)
    \propto
    \pi_{\mathrm{ref}}(y\mid x)
    \exp\!\left(
        \frac{
            g_f^{-1}(f'(R_x(y)))
        }{\beta}
    \right),
\end{equation}
which proves the result.
\end{proof}
\subsection{Proof of Theorem~\ref{thm:rank-consistency}}
\label{app:proof-rank-consistency}

\begin{proof}
Fix a prompt $x$.

\paragraph{Implicit-reward ordering.}
By Lemma~\ref{lem:density_ratio_util}, rank-favored weights imply that
\begin{equation}
    u_x(y)>u_x(y')
    \quad\Longrightarrow\quad
    R_x(y)\ge R_x(y').
    \label{eq:app-density-ratio-order}
\end{equation}
From Theorem~\ref{thm:divergence-alignment},
\begin{equation}
    r_{\theta^*}(x,y)
    =
    g_f^{-1}\!\left(f'(R_x(y))\right).
\end{equation}
Since $f'$ is nondecreasing for convex $f$ and $g_f^{-1}$ is
monotone,
Eq.~\eqref{eq:app-density-ratio-order} implies
\[
    u_x(y)>u_x(y')
    \quad\Longrightarrow\quad
    r_{\theta^*}(x,y)
    \ge
    r_{\theta^*}(x,y').
\]
Applying this relation to
$y_{(1)}\succ_x\cdots\succ_x y_{(k)}$ gives
\[
    r_{\theta^*}(x,y_{(1)})
    \ge
    \cdots
    \ge
    r_{\theta^*}(x,y_{(k)}).
\]

\paragraph{Policy ordering under reference bias.}
For adjacent ranks $i$ and $i+1$, Theorem~
\ref{thm:divergence-alignment} gives
\begin{align}
\log
\frac{
    \pi_{\theta^*}(y_{(i)}\mid x)
}{
    \pi_{\theta^*}(y_{(i+1)}\mid x)
}
&=
\log
\frac{
    \pi_{\mathrm{ref}}(y_{(i)}\mid x)
}{
    \pi_{\mathrm{ref}}(y_{(i+1)}\mid x)
}
\nonumber\\
&\quad+
\frac{1}{\beta}
\Big[
    g_f^{-1}(f'(R_x(y_{(i)})))
    -
    g_f^{-1}(f'(R_x(y_{(i+1)})))
\Big].
\end{align}
Using the definitions of $\Delta_{x,i}$ and $b_{x,i}$ from
Theorem~\ref{thm:rank-consistency}, this becomes
\begin{equation}
\log
\frac{
    \pi_{\theta^*}(y_{(i)}\mid x)
}{
    \pi_{\theta^*}(y_{(i+1)}\mid x)
}
=
-b_{x,i}
+
\frac{\Delta_{x,i}}{\beta}.
\label{eq:app-policy-rank-log-odds}
\end{equation}

Therefore, the desired adjacent ordering
\[
    \pi_{\theta^*}(y_{(i)}\mid x)
    \ge
    \pi_{\theta^*}(y_{(i+1)}\mid x)
\]
holds whenever
\[
    \frac{\Delta_{x,i}}{\beta}
    \ge
    b_{x,i}.
\]

If $b_{x,i}\le0$, this condition is automatically satisfied because
$\Delta_{x,i}\ge0$. If $b_{x,i}>0$, it is sufficient that
\[
    \beta
    \le
    \frac{\Delta_{x,i}}{b_{x,i}}.
\]
Hence the ordering holds simultaneously for all prompts and adjacent
ranks whenever
\begin{equation}
    \beta
    \le
    \inf_{x,\;i:\,b_{x,i}>0}
    \frac{\Delta_{x,i}}{b_{x,i}},
\end{equation}
provided the right-hand side is positive.

Since ordering every adjacent pair implies the complete ordering,
\[
    \pi_{\theta^*}(y_{(1)}\mid x)
    \ge
    \cdots
    \ge
    \pi_{\theta^*}(y_{(k)}\mid x),
\]
which completes the proof.
\end{proof}
\subsection{Rank-Favored Property of the Experimental Weights}
\label{app:rank-weight-condition}

In our experiments, we use the normalized weights
\begin{equation}
    \rho_x(i)
    =
    \frac{2(k-i)}{k(k-1)},
    \qquad
    i=1,\ldots,k-1,
\end{equation}
and
\begin{equation}
    q_x(j\mid i)
    =
    \frac{2(j-i)}
    {(k-i)(k-i+1)},
    \qquad
    j=i+1,\ldots,k.
\end{equation}

The corresponding aligned rank mass is
\begin{equation}
    B_i
    =
    \frac{2(k-i)}{k(k-1)},
    \qquad i<k,
\end{equation}
which is nonincreasing in $i$.

For $j>1$, the unaligned rank mass is
\begin{align}
    A_j
    &=
    \sum_{i<j}
    \rho_x(i)q_x(j\mid i)
    \\
    &=
    \frac{4}{k(k-1)}
    \sum_{i=1}^{j-1}
    \frac{j-i}{k-i+1}.
    \label{eq:app-experimental-Aj}
\end{align}
Moreover,
\begin{align}
    A_{j+1}-A_j
    &=
    \frac{4}{k(k-1)}
    \sum_{i=1}^{j}
    \frac{1}{k-i+1}
    >0,
\end{align}
so $A_j$ is increasing in rank index.

Hence, for $i<j$,
\[
    B_i\ge B_j,
    \qquad
    A_i\le A_j,
\]
and therefore
\begin{equation}
    B_iA_j
    \ge
    B_jA_i.
\end{equation}
Thus the weighting distributions used in our experiments satisfy the
rank-favored condition in
Eq.~\eqref{eq:rank-weight-condition}.
\section{Ranked Data Generation}
\label{sec: ranked_data}
\paragraph{Data design.}
We construct a five-way preference dataset in which the contrast
between ranks is conspicuous.  Unlike a binary preference corpus,
each example contains a prompt $x$ and five responses
$y_1 \succ y_2 \succ \cdots \succ y_5$.  The candidates are generated by 5 different models/configs of varying capabilities without
an assigned rank; their ordering is induced only after generation by an
external reward model. We make sure these generations are diverse by making sure the reward gaps between adjacent ranks are wide enough.

\subsection{Prompt sampling and splits}
We sample 4,000 diverse prompts from the prompt pool of
\texttt{openbmb/UltraFeedback}~\citep{cui2024ultra}.  Sampling involves 
 1,200 FLAN, 900 Evol-Instruct, 700 UltraChat, 700 ShareGPT, and 500 FalseQA prompts.  TruthfulQA prompts are excluded to reduce
direct overlap with a frequently used factuality benchmark. Duplicates are removed, and the prompts are divided into 3,000 training, 400
validation, and 600 test prompts \emph{before} any response is generated.

\subsection{Diverse Response Generation}

For every prompt, we generate one response from each model variant in
Table~\ref{tab:appendix-generation-mixture}, reflecting diverse set of responses of varying utility.  The mixture is deliberately
heterogeneous,  Qwen2.5-14B-Instruct supplies a relatively strong and stable
candidate; the two Qwen2.5-7B-Instruct slots separate ordinary and more diverse
sampling; Llama-3.1-8B-Instruct introduces cross-family variation; and
Qwen2.5-1.5B-Instruct supplies a naturally weaker candidate.

\begin{table}[!ht]
  \centering
  \small
  \caption{Models and decoding parameters used to generate the five response
  candidates. }
  \label{tab:appendix-generation-mixture}
  \begin{tabular}{lccc}
    \toprule
    Generator & Temperature & Top-$p$ & Max. new tokens \\
    \midrule
    Qwen2.5-14B-Instruct     & 0.25 & 0.90 & 1,536 \\
    Qwen2.5-7B-Instruct      & 0.70 & 0.90 & 1,536 \\
    Qwen2.5-7B-Instruct      & 1.00 & 0.95 & 1,536 \\
    Llama-3.1-8B-Instruct    & 0.70 & 0.90 & 1,536 \\
    Qwen2.5-1.5B-Instruct    & 0.80 & 0.95 & 1,536 \\
    \bottomrule
  \end{tabular}
\end{table}

\subsection{Reward scoring and rank construction}

We score each prompt--response conversation independently with \texttt{Skywork Reward
V2 (Llama-3.1-8B)}~\citep{liu2024skywork}. For a prompt $x$ and response $y_i$, let
$a_i=R_{\mathrm{Skywork}}(x,y_i)$ be the scalar reward.  The ground-truth
ranking is formed by descending reward,
\begin{equation}
  \operatorname{argsort}_{i\in\{1,\ldots,5\}}(-a_i),
  \qquad
  a_{\pi_{x,1}} \geq \cdots \geq a_{\pi_{x,5}}.
\end{equation}
Our Alignment methods use these ordinal ranks as ground truth.

Table~\ref{tab:appendix-generator-statistics} shows that the varying response generators from Table \ref{tab:appendix-generation-mixture} create the intended heirarchical utility of responses.  In particular, the weak 1.5B candidate is the worst; the 14B model with conservative sampling is the best in terms of mean reward/rank, while the other generator configs lie in order as well.

\begin{table}[!ht]
  \centering
  \small
  \caption{Reward and rank statistics before high-confidence filtering,
  computed over the 3,854 complete five-way rows.  Lower mean rank is better.}
  \label{tab:appendix-generator-statistics}
  \begin{tabular}{lrrrr}
    \toprule
    Candidate & Mean reward & Mean rank & Top-1 (\%) & Bottom (\%) \\
    \midrule
    Qwen 14B, conservative & 18.67 & 2.31 & 31.76 & 2.83 \\
    Qwen 7B, balanced      & 17.58 & 2.55 & 22.08 & 3.30 \\
    Llama 8B, balanced     & 17.32 & 2.68 & 28.23 & 8.85 \\
    Qwen 7B, diverse       & 16.74 & 2.82 & 16.19 & 5.27 \\
    Qwen 1.5B, weaker      &  4.02 & 4.64 &  1.74 & 79.76 \\
    \bottomrule
  \end{tabular}
\end{table}

Across all 19,270 responses, the reward median is 14.56, with 5th and 95th
percentiles $-4.61$ and $36.00$.  The median within-prompt top-to-bottom gap is
17.93 (5th percentile 4.88; 95th percentile 38.25), so the best--worst
contrast is typically strong.  Adjacent ranks are less certain: 444 complete
rows (11.52\%) contain an exact reward tie.  Exact equality is not rare because
the reward model is evaluated in bfloat16; retaining these rows would be against our design principle of having conspicuous ranking across diverse responses. Hence, we keep only rows which have a high margin of rewards across adjacent ranks and no ties.

\subsection{High-confidence filtering}

For the reward-sorted candidates, define the minimum adjacent margin
\begin{equation}
  \Delta_{\min}(x)
  =\min_{j\in\{1,\ldots,4\}}
  \left(a_{\pi_{x,j}}-a_{\pi_{x,j+1}}\right).
\end{equation}
The experiments in the main paper use only rows with no exact reward tie and
$\Delta_{\min}(x)\geq0.5$.  This conservative rule retains 2,279 prompts:
1,675 train, 235 validation, and 369 test.  Every retained prompt has five
complete responses and a strict total order.  Table~\ref{tab:appendix-data-composition}
reports the final source composition. This dataset should be understood as a high quality filtered data where each prompt is diverse and response ranking makes sense qualitatively.

\begin{table}[t]
  \centering
  \small
  \caption{Composition of the high-confidence dataset used in all main
  experiments.}
  \label{tab:appendix-data-composition}
  \begin{tabular}{lrrrr}
    \toprule
    Source & Train & Validation & Test & Total \\
    \midrule
    Evol-Instruct & 395 & 53 & 82 & 530 \\
    FLAN           & 365 & 59 & 93 & 517 \\
    FalseQA        & 253 & 33 & 57 & 343 \\
    ShareGPT       & 303 & 49 & 62 & 414 \\
    UltraChat      & 359 & 41 & 75 & 475 \\
    \midrule
    Total          & 1,675 & 235 & 369 & 2,279 \\
    \bottomrule
  \end{tabular}
\end{table}

\section{Training Details}
\label{sec: training_details}
\paragraph{Models and parameter-efficient finetuning.}
We train the non-instruction-tuned base models using \texttt{Unsloth 2025.10.1} \citep{unsloth} and
LoRA \citep{hu2021lora}.  We use the following base models for our alignment training: Qwen2.5-0.5B, Qwen2.5-1.5B,
Qwen2.5-7B,Gemma-2-2B, Llama-3.2-3B, Llama-3.1-8B, and Gemma-2-9B. 
We start with base models to isolate the effect and comparison to the alignment method itself, with minimal instruction-based prior knowledge.

For every model, LoRA is applied to the query, key, value, and output
attention projections and to the gate, up, and down MLP projections.  We use
rank 64, scaling $\alpha=64$, zero LoRA dropout, no bias parameters, and
Unsloth gradient checkpointing.  Models are loaded and trained in bfloat16.  Each run uses a single H200 GPU. Hyperparameter details are provided in Table \ref{tab:appendix-training-hparams}. Refer to the code repository \href{https://github.com/rhaldarpurdue/f-RPO}{[link]} for the exact implementations.

\begin{table}[!ht]
  \centering
  \small
  \caption{Training hyperparameters shared across methods unless noted.}
  \label{tab:appendix-training-hparams}
  \begin{tabular}{ll}
    \toprule
    Hyperparameter & Value \\
    \midrule
    Training epochs & 1 \\
    Random seeds & $\{1,2,3,4,5\}$ \\
    Learning rate & $5\times10^{-5}$ \\
    Optimizer & 8-bit AdamW \\
    Weight decay & 0.1 \\
    LR schedule / warmup & cosine / 0.1 fraction \\
    KL penalty coefficient, $\beta$ & 0.1 \\
    Maximum gradient norm & 1.0 \\
    Precision & bfloat16 \\
    LoRA rank / $\alpha$ & 64 / 64 \\
    LoRA target modules & $q,k,v,o$, gate, up, down projections \\
    Ranked-method batch / accumulation & 1 / 16 \\
    Pairwise-DPO batch / accumulation & 10 / 16 \\
    \bottomrule
  \end{tabular}
\end{table}

\section{Evaluation Methodology}
We evaluate two classes of metrics, as described in the following 2 sections, corresponding to rank \& generation quality.
\label{sec: evaluation methodology}
\subsection{Ranking Metrics}
\label{app:ranking-methodology}

We evaluate whether the aligned policy ranks a fixed set of
$K=5$ candidate responses in the same order as the ground-truth
preference ranking. For a candidate response
$y_i=(y_{i,1},\ldots,y_{i,L_i})$, we define its policy score using
length-normalized log-likelihood,
\begin{equation}
    s_i
    =
    \frac{1}{L_i}
    \sum_{t=1}^{L_i}
    \log
    \pi_\theta
    \!\left(
        y_{i,t}\mid x,y_{i,<t}
    \right).
    \label{eq:appendix-candidate-score}
\end{equation}
Length normalization mitigates the systematic preference of summed
log-likelihood toward shorter responses. The predicted rank
$\widehat r_i$ is obtained by sorting $s_i$ in descending order,
while the target rank $r_i\in\{1,\ldots,K\}$ is determined by the
reward-model ordering. Ties are assigned dense ranks.

For each prompt, we compute the following metrics and report their
average across prompts.

\begin{itemize}

\item \textbf{Spearman's rank correlation.}
Our primary global-ranking metric is the Pearson correlation between
the target and predicted rank vectors:
\begin{equation}
    \rho
    =
    \frac{
        \sum_{i=1}^{K}
        (r_i-\bar r)
        (\widehat r_i-\bar{\widehat r})
    }{
        \sqrt{
            \sum_{i=1}^{K}(r_i-\bar r)^2
        }
        \sqrt{
            \sum_{i=1}^{K}
            (\widehat r_i-\bar{\widehat r})^2
        }
    },
    \label{eq:spearman}
\end{equation}
where
$\bar r=K^{-1}\sum_i r_i$ and
$\bar{\widehat r}=K^{-1}\sum_i\widehat r_i$.
Spearman's $\rho\in[-1,1]$ measures monotonic agreement over the
entire candidate ranking, with larger values indicating better
agreement.

\item \textbf{Kendall's $\tau$.}
We additionally report Kendall's $\tau_b$, which measures agreement
through pairwise concordance while accounting for possible ties.
Let $C$ and $D$ denote the numbers of concordant and discordant
candidate pairs, respectively, and let $T_r$ and
$T_{\widehat r}$ denote the numbers of pairs tied only in the target
and predicted rankings. Then
\begin{equation}
    \tau_b
    =
    \frac{C-D}
    {\sqrt{(C+D+T_r)(C+D+T_{\widehat r})}}.
    \label{eq:kendall-tau}
\end{equation}
In the absence of ties, this reduces to
$\tau=(C-D)/\binom{K}{2}$.

\item \textbf{Pairwise accuracy.}
We measure the fraction of ground-truth ordered pairs whose ordering
is recovered by the policy scores:
\begin{equation}
    \mathrm{PairAcc}
    =
    \frac{1}{|\mathcal P_x|}
    \sum_{(i,j)\in\mathcal P_x}
    \mathbf{1}\{s_i>s_j\},
    \qquad
    \mathcal P_x
    =
    \{(i,j):r_i<r_j\}.
    \label{eq:pairwise-accuracy}
\end{equation}
Ground-truth ties, if present, are excluded; the high-confidence test
set used in our experiments contains no such ties.

\item \textbf{Top-1 accuracy.}
Top-1 accuracy measures whether the policy identifies the
ground-truth highest-ranked response:
\begin{equation}
    \mathrm{Top1}
    =
    \mathbf{1}
    \left\{
        \arg\max_i s_i
        =
        \arg\min_i r_i
    \right\}.
    \label{eq:top1-accuracy}
\end{equation}

\item \textbf{Exact ranking match.}
This metric equals one only when the complete predicted ranking
matches the target ranking:
\begin{equation}
    \mathrm{Exact}
    =
    \mathbf{1}
    \left\{
        \widehat{\mathbf r}
        =
        \mathbf r
    \right\}.
    \label{eq:exact-ranking}
\end{equation}

\item \textbf{Rank MAE and RMSE.}
Finally, we measure absolute and squared rank displacement:
\begin{align}
    \mathrm{MAE}
    &=
    \frac{1}{K}
    \sum_{i=1}^{K}
    |\widehat r_i-r_i|,
    \\
    \mathrm{RMSE}
    &=
    \sqrt{
        \frac{1}{K}
        \sum_{i=1}^{K}
        (\widehat r_i-r_i)^2
    }.
\end{align}
Lower values indicate better recovery of the target ranking.

\end{itemize}
\subsection{AlpacaEval 2.0 generation evaluation}
\label{app:alpacaeval-methodology}
For every model, we evaluate the unaligned base model and
aligned checkpoints using the same 805 AlpacaEval \citep{alpaca} instructions.  Generation uses vLLM with greedy decoding
(temperature 0, top-$p=1$), seed 3407, a maximum context length of 4,096, and
at most 1,024 new tokens.  Greedy decoding eliminates sampling variance and
makes differences attributable to the learned policy rather than decoding
noise.

Each aligned model is compared
against outputs from its corresponding unaligned base model.  This
base-relative win rate isolates the utility gained by preference alignment and
is especially interpretable for small base models. Pairwise judgments are produced by
\texttt{gpt-4.1-mini-2025-04-14} with temperature 0 and a structured function
call that must return an ordering of the two anonymized responses.  The judge
is instructed to prioritize correctness, instruction following, relevance,
clarity, and appropriate detail and not to prefer an answer merely because it
is longer.  We report raw win rate.

\newpage


\FloatBarrier
\section{Additional Tables}
\label{app: additional_tables}


\begin{table*}[!ht]
\centering
\small
\setlength{\tabcolsep}{5pt}
\caption{
Expanded ranking performance across alignment method for Qwen2.5-1.5B.
Parentheses denote standard errors over five seeds;
the base model is evaluated deterministically.
Best and second-best aligned results in each column are
\textbf{bold} and \underline{underlined}, respectively.
Shaded rows denote our $f$-RPO variants.
}
\label{tab:scaling-k5-test-qwen2515b}

\resizebox{\textwidth}{!}{%
\begin{tabular}{lccccccc}
\toprule
Method
& Spearman $\rho$ ($\uparrow$)
& Kendall $\tau$ ($\uparrow$)
& Pairwise (\%) ($\uparrow$)
& Top-1 (\%) ($\uparrow$)
& Exact match (\%) ($\uparrow$)
& MAE ($\downarrow$)
& RMSE ($\downarrow$) \\
\midrule

\rowcolor{gray!15}
$f$-RPO-RKL
& \textbf{0.1538 (0.0068)}
& \textbf{0.1316 (0.0058)}
& \textbf{56.58 (0.29)}
& \textbf{27.32 (0.10)}
& \underline{2.82 (0.11)}
& \textbf{1.4298 (0.0088)}
& \textbf{1.7224 (0.0085)} \\

\rowcolor{gray!15}
$f$-RPO-TV
& \underline{0.1288 (0.0035)}
& \underline{0.1079 (0.0026)}
& \underline{55.39 (0.13)}
& \underline{26.72 (0.24)}
& \textbf{2.87 (0.07)}
& \underline{1.4448 (0.0037)}
& \underline{1.7521 (0.0040)} \\

PL-DPO
& 0.0471 (0.0045)
& 0.0441 (0.0036)
& 52.21 (0.18)
& 24.66 (0.30)
& 2.38 (0.05)
& 1.5261 (0.0047)
& 1.8431 (0.0052) \\

Mallows-RCPO
& 0.0394 (0.0111)
& 0.0375 (0.0090)
& 51.88 (0.45)
& 24.12 (0.59)
& 2.38 (0.05)
& 1.5343 (0.0115)
& 1.8533 (0.0122) \\

\midrule
Base
& -0.0431
& -0.0287
& 48.56
& 20.33
& 1.36
& 1.6260
& 1.9557 \\

\bottomrule
\end{tabular}%
}
\end{table*}

\begin{table*}[!ht]
\centering
\small
\setlength{\tabcolsep}{5pt}
\caption{
Expanded ranking performance across alignment method for Qwen2.5-7B.
Parentheses denote standard errors over five seeds;
the base model is evaluated deterministically.
Best and second-best aligned results in each column are
\textbf{bold} and \underline{underlined}, respectively.
Shaded rows denote our $f$-RPO variants.
}
\label{tab:scaling-k5-test-qwen257b}

\resizebox{\textwidth}{!}{%
\begin{tabular}{lccccccc}
\toprule
Method
& Spearman $\rho$ ($\uparrow$)
& Kendall $\tau$ ($\uparrow$)
& Pairwise (\%) ($\uparrow$)
& Top-1 (\%) ($\uparrow$)
& Exact match (\%) ($\uparrow$)
& MAE ($\downarrow$)
& RMSE ($\downarrow$) \\
\midrule

\rowcolor{gray!15}
$f$-RPO-RKL
& \underline{0.4827 (0.0008)}
& \underline{0.3951 (0.0009)}
& \underline{69.76 (0.04)}
& \underline{29.76 (0.20)}
& \textbf{5.09 (0.18)}
& \textbf{1.0363 (0.0012)}
& \underline{1.3261 (0.0018)} \\

\rowcolor{gray!15}
$f$-RPO-TV
& \textbf{0.4858 (0.0005)}
& \textbf{0.3989 (0.0006)}
& \textbf{69.95 (0.03)}
& 28.40 (0.10)
& \underline{4.82 (0.13)}
& \underline{1.0383 (0.0015)}
& \textbf{1.3260 (0.0018)} \\

PL-DPO
& 0.4530 (0.0006)
& 0.3708 (0.0004)
& 68.54 (0.02)
& \textbf{30.46 (0.14)}
& 4.39 (0.16)
& 1.0810 (0.0011)
& 1.3680 (0.0010) \\

Mallows-RCPO
& 0.4674 (0.0023)
& 0.3828 (0.0023)
& 69.14 (0.12)
& \underline{29.76 (0.23)}
& 4.28 (0.10)
& 1.0606 (0.0036)
& 1.3451 (0.0037) \\

\midrule
Base
& 0.4390
& 0.3588
& 67.94
& 29.27
& 2.98
& 1.1003
& 1.3919 \\

\bottomrule
\end{tabular}%
}
\end{table*}

\begin{table*}[!ht]
\centering
\small
\setlength{\tabcolsep}{5pt}
\caption{
 Expanded $f$-divergence ablation for ranked alignment on Qwen2.5-0.5B base model.
Parentheses denote standard errors over five seeds; the base model is
evaluated deterministically.
Best and second-best aligned results in each column are
\textbf{bold} and \underline{underlined}, respectively.
For MAE and RMSE, lower is better.
}
\label{tab:frpo-divergence-k5-all-test}

\resizebox{\textwidth}{!}{%
\begin{tabular}{lccccccc}
\toprule
Divergence
& Spearman $\rho$
& Kendall $\tau$
& Pairwise acc.$\%$.
& Top-1 (\%)
& Exact match
& MAE ($\downarrow$)
& RMSE ($\downarrow$)\\
\midrule

Reverse KL
& \textbf{0.1648 (0.0039)}
& \textbf{0.1340 (0.0031)}
& \textbf{56.70 (0.16)}
& \textbf{28.94 (0.10)}
& \textbf{0.0320 (0.0005)}
& \textbf{1.4131 (0.0053)}
& \textbf{1.7066 (0.0048)} \\

Total variation
& \underline{0.1492 (0.0017)}
& \underline{0.1233 (0.0016)}
& \underline{56.16 (0.08)}
& \underline{28.08 (0.07)}
& 0.0309 (0.0007)
& \underline{1.4333 (0.0019)}
& \underline{1.7252 (0.0020)} \\

KL
& 0.0928 (0.0019)
& 0.0796 (0.0017)
& 53.98 (0.08)
& 26.94 (0.29)
& 0.0298 (0.0000)
& 1.4955 (0.0021)
& 1.7899 (0.0019) \\

Pearson $\chi^2$
& 0.1048 (0.0012)
& 0.0879 (0.0010)
& 54.40 (0.05)
& 27.70 (0.26)
& 0.0304 (0.0010)
& 1.4784 (0.0020)
& 1.7749 (0.0019) \\

Squared Hellinger
& 0.0937 (0.0005)
& 0.0792 (0.0006)
& 53.96 (0.03)
& 27.80 (0.18)
& \underline{0.0314 (0.0018)}
& 1.4862 (0.0011)
& 1.7861 (0.0015) \\

Jensen--Shannon
& 0.1044 (0.0010)
& 0.0881 (0.0011)
& 54.41 (0.05)
& 27.97 (0.16)
& \underline{0.0314 (0.0011)}
& 1.4782 (0.0012)
& 1.7745 (0.0013) \\

\midrule
Base
& 0.0314
& 0.0320
& 51.60
& 23.85
& 0.0244
& 1.5631
& 1.8644 \\

\bottomrule
\end{tabular}%
}
\end{table*}

\FloatBarrier
\newpage
\section{Qualitative Examples}
\label{app:qualitative}

We provide representative AlpacaEval generations from Qwen2.5-1.5B
aligned via different methods. The examples illustrate qualitative
differences between alignment objectives.

\begin{table}[!ht]
\centering
\caption{
Example 1: Inappropriate refusal on a benign question.
}
\label{tab:qualitative-example1}
\scriptsize
\setlength{\tabcolsep}{4pt}
\renewcommand{\arraystretch}{1.05}

\begin{adjustbox}{
    max totalsize={\textwidth}{0.80\textheight},
    keepaspectratio,
    center
}
\begin{tabular}{
    >{\raggedright\arraybackslash\bfseries}p{0.13\textwidth}
    >{\raggedright\arraybackslash}p{0.81\textwidth}
}
\toprule

Prompt
&
\emph{Why do a lot of Scientists not believe in God or Satan?}
\\
\midrule

Base
&
\textit{Excerpt:}
``I'm sorry, but I'm not able to provide answers to personal or
philosophical questions.''

\smallskip
\textbf{Generation:}
I'm sorry, but I'm not able to provide answers to personal or
philosophical questions. If you have any other questions that I can
assist you with, please feel free to ask.
\\
\midrule

PL-DPO
&
\textit{Excerpt:}
``I'm sorry, but I can't assist with that.''

\smallskip
\textbf{Generation:}
I'm sorry, but I can't assist with that.
\\
\midrule

Mallows-RCPO
&
\textit{Excerpt:}
Explains that scientists hold diverse religious and philosophical
beliefs and discusses several possible factors.

\smallskip
\textbf{Generation:}
Scientists, like any other human beings, have their own beliefs and
perspectives on the existence of God or Satan. While some scientists
may hold religious beliefs, others may not, and some may even be
atheists or agnostics. The reasons for this diversity of beliefs can
be attributed to several factors:

\textbf{1. Scientific Method and Objectivity:}
The scientific method is based on empirical evidence and the pursuit
of truth through observation and experimentation. This approach
inherently rejects supernatural explanations, as they are not subject
to empirical testing.

\textbf{2. Philosophical and Theological Differences:}
Different scientific disciplines may be influenced by different
philosophical and theological traditions.

\textit{[\ldots intermediate discussion omitted for space \ldots]}

In summary, the reasons why a lot of scientists do not believe in God
or Satan are complex and multifaceted.
\\
\midrule

$f$-RPO-RKL
&
\textit{Excerpt:}
``The belief in God or Satan is a deeply personal and subjective
matter \ldots'' and discusses empirical, philosophical, and cultural
factors.

\smallskip
\textbf{Generation:}
The belief in God or Satan is a deeply personal and subjective matter,
influenced by a variety of factors including cultural, religious,
philosophical, and personal experiences. While some scientists may not
believe in God or Satan, it's important to understand that this is not
a universal or consistent stance among scientists.

\textbf{1. Scientific Method and Evidence:}
Scientists are trained to seek evidence and test hypotheses through
empirical observation and experimentation.

\textbf{2. Philosophical Skepticism:}
Many scientists are skeptical of claims that cannot be empirically
tested or verified.

\textit{[\ldots intermediate discussion omitted for space \ldots]}

The scientific community is diverse, and there is no one-size-fits-all
approach to belief.
\\
\midrule

$f$-RPO-TV
&
\textit{Excerpt:}
``Scientists, like any other human beings, have a wide range of beliefs
and worldviews \ldots'' and discusses methodological naturalism,
empirical evidence, and philosophical skepticism.

\smallskip
\textbf{Generation:}
Scientists, like any other human beings, have a wide range of beliefs
and worldviews, and their views on God or Satan can vary widely.

\textbf{1. Philosophical and Theological Skepticism:}
Many scientists are deeply skeptical of religious claims and the
supernatural.

\textbf{2. Methodological Naturalism:}
Scientists often adhere to methodological naturalism, the view that
phenomena in the natural world can be explained by natural causes
without requiring supernatural explanations.

\textbf{3. Scientific Method and Empirical Evidence:}
Scientists rely on empirical evidence and the scientific method to
understand the natural world.

\textit{[\ldots intermediate discussion omitted for space \ldots]}

In summary, while scientists are not a monolithic group, many of them
do not believe in God or Satan due to a combination of philosophical,
theological, and scientific reasons.
\\

\bottomrule
\end{tabular}
\end{adjustbox}
\end{table}
\begin{table}[!ht]
\centering
\caption{
Example 2: Incomplete answer and repetitive degeneration.
}
\label{tab:qualitative-example2}
\scriptsize
\setlength{\tabcolsep}{4pt}
\renewcommand{\arraystretch}{1.05}

\begin{adjustbox}{
    max totalsize={\textwidth}{0.80\textheight},
    keepaspectratio,
    center
}
\begin{tabular}{
    >{\raggedright\arraybackslash\bfseries}p{0.13\textwidth}
    >{\raggedright\arraybackslash}p{0.81\textwidth}
}
\toprule

Prompt
&
\emph{Is JavaScript a back-end or front-end language?}
\\
\midrule

Base
&
\textit{Excerpt:}
``javascript is a front end language.''

\smallskip
\textbf{Generation:}
javascript is a front end language.
\\
\midrule

PL-DPO
&
\textit{Excerpt:}
States that JavaScript is primarily front-end, later acknowledges
server-side use, and then repeats the same passage multiple times.

\smallskip
\textbf{Generation:}
is javascript back end or front end language? JSName is a scripting
language that is used for web development. It is a front-end language,
which means it is used to create the user interface of a web
application. JavaScript is not a back-end language, which refers to the
part of a web application that is responsible for processing and
managing data.

JavaScript is primarily used for client-side scripting, which means it
is executed on the user's web browser. It is used to add interactivity
and dynamic behavior to web pages, such as form validation, animations,
and user interactions. JavaScript is also used for server-side
scripting, but it is not the primary language used for server-side
development.

In summary, JavaScript is a front-end language used for web
development, while back-end languages are used for server-side
development. JavaScript is not a back-end language, but it can be used
for server-side scripting in conjunction with other languages like
Node.js.

Can you give me some examples of JavaScript frameworks that are
commonly used for back-end development?

\textit{[\ldots the preceding passage and question are repeated
multiple times until the generation limit \ldots]}
\\
\midrule

Mallows-RCPO
&
\textit{Excerpt:}
``JavaScript is a versatile programming language that can be used for
both the back end and the front end of web applications.''

\smallskip
\textbf{Generation:}
JavaScript is a versatile programming language that can be used for
both the back end and the front end of web applications. Here's a
breakdown of its use in each:

\textbf{Front End (Client-Side).}
JavaScript runs in the user's browser and is used for dynamic content,
interactivity, event handling, data binding, and asynchronous requests.

\textbf{Back End (Server-Side).}
JavaScript can also run on the server through Node.js. Express.js can
be used to create web applications and RESTful APIs, while server-side
JavaScript can handle requests, process data, and interact with
databases.

\textit{[\ldots additional examples of React, Angular, Vue.js,
Node.js, Express.js, and server-side rendering omitted for space
\ldots]}

In summary, JavaScript is a versatile language that can be used for
both the front end and the back end of web applications, depending on
the context and the framework or technology being used.
\\
\midrule

$f$-RPO-RKL
&
\textit{Excerpt:}
``JavaScript is a versatile programming language that can be used for
both the front-end and back-end of web applications.''

\smallskip
\textbf{Generation:}
JavaScript is a versatile programming language that can be used for
both the front-end and back-end of web applications. Here's a brief
overview of how it's used in each context:

\textbf{Front-End (Client-Side).}
Its purpose is to create the user interface and handle user
interactions, with technologies and frameworks such as HTML, CSS,
React, Angular, Vue.js, and jQuery.

\textbf{Back-End (Server-Side).}
It can be used to handle server-side logic, database interactions, and
other server-side tasks, particularly through Node.js and Express.js.

\textit{[\ldots a longer worked example containing HTML and Node.js
code is omitted for space \ldots]}

In summary, JavaScript is a versatile language that can be used for
both front-end and back-end development, depending on the specific
requirements of the application.
\\
\midrule

$f$-RPO-TV
&
\textit{Excerpt:}
``Javascript is a programming language that can be used for both the
back end and the front end of web applications.''

\smallskip
\textbf{Generation:}
Javascript is a programming language that can be used for both the
back end and the front end of web applications.

\textbf{Front End (Client-Side).}
JavaScript is used for handling user interactions, rendering dynamic
content, and managing application behavior. Frameworks and libraries
such as React, Angular, and Vue.js are commonly used for front-end
development.

\textbf{Back End (Server-Side).}
JavaScript can also be used on the back end through Node.js and
frameworks such as Express.js, including for server-side rendering and
APIs.

\textit{[\ldots additional examples omitted for space \ldots]}

In summary, while JavaScript is a versatile language that can be used
for both front end and back end development, it is more commonly used
for the front end due to its simplicity and ease of use in creating
dynamic and interactive user interfaces.
\\

\bottomrule
\end{tabular}
\end{adjustbox}
\end{table}
\begin{table}[t]
\centering
\caption{
Example 3: Degeneration under prior ranked objectives.
}
\label{tab:qualitative-example3}
\scriptsize
\setlength{\tabcolsep}{4pt}
\renewcommand{\arraystretch}{1.05}

\begin{adjustbox}{
    max totalsize={\textwidth}{0.78\textheight},
    keepaspectratio,
    center
}
\begin{tabular}{
    >{\raggedright\arraybackslash\bfseries}p{0.13\textwidth}
    >{\raggedright\arraybackslash}p{0.81\textwidth}
}
\toprule

Prompt
&
\emph{What breed dog is smallest?}
\\
\midrule

Base
&
\textit{Excerpt:}
Produces an unrelated question.

\smallskip
\textbf{Generation:}
What is the name of the character who is a former soldier?
\\
\midrule

PL-DPO
&
\textit{Excerpt:}
Echoes the prompt and degenerates into repeated non-English tokens.

\smallskip
\textbf{Generation:}
What breed dog is smallest?

\textit{[\ldots a non-English token is repeated until the generation
limit \ldots]}
\\
\midrule

Mallows-RCPO
&
\textit{Excerpt:}
Changes the question and then degenerates into repeated non-English
tokens.

\smallskip
\textbf{Generation:}
What is the name of the dog that is a cross between a beagle and a
terrier?

\textit{[\ldots a non-English token is repeated until the generation
limit \ldots]}
\\
\midrule

$f$-RPO-RKL
&
\textit{Excerpt:}
``The breed of dog that is considered the smallest is the Chihuahua.''

\smallskip
\textbf{Generation:}
The breed of dog that is considered the smallest is the Chihuahua.
The Chihuahua is a small dog breed that originated in Mexico,
specifically in the state of Chihuahua. It is known for its small size,
which can range from 6 to 10 inches (15 to 25 centimeters) in height
and weighing between 3 to 10 pounds (1.4 to 4.5 kilograms).

The Chihuahua is a popular breed for many reasons, including its small
size, friendly nature, and the fact that it can be trained to perform
various tasks, such as being a companion, a watchdog, or even a
therapy dog.

Other small dog breeds include the Maltese, Pomeranian, and Yorkshire
Terrier, but the Chihuahua is the smallest of these breeds.
\\
\midrule

$f$-RPO-TV
&
\textit{Excerpt:}
``The smallest dog breed is the Chihuahua.''

\smallskip
\textbf{Generation:}
The smallest dog breed is the Chihuahua. However, it's important to
note that the smallest Chihuahua can be quite small, with some
individuals being as small as 3--4 inches (7.6--10.2 cm) in height
and weighing less than 1 pound (0.45 kg).

For other breeds, the smallest can vary widely depending on the breed
and the individual dog. For example, the response discusses the
Pomeranian, Bichon Frise, and Pug as other small breeds.

\textit{[\ldots intermediate breed comparisons omitted for space
\ldots]}

These are just the extremes, and there are many other breeds that can
be even smaller. It's always a good idea to consult with a veterinarian
or a reputable breeder to get accurate information about the size of a
specific dog breed.
\\

\bottomrule
\end{tabular}
\end{adjustbox}
\end{table}
\begin{table}[t]
\centering
\caption{
Example 4: Distinguishing an abstract data type from its implementation.
}
\label{tab:qualitative-example4}
\scriptsize
\setlength{\tabcolsep}{4pt}
\renewcommand{\arraystretch}{1.05}

\begin{adjustbox}{
    max totalsize={\textwidth}{0.80\textheight},
    keepaspectratio,
    center
}
\begin{tabular}{
    >{\raggedright\arraybackslash\bfseries}p{0.13\textwidth}
    >{\raggedright\arraybackslash}p{0.81\textwidth}
}
\toprule

Prompt
&
\emph{Is queue an ADT or a data structure?}
\\
\midrule

Base
&
\textit{Excerpt:}
Echoes the prompt and degenerates into token repetition.

\smallskip
\textbf{Generation:}
is queue an ADT or a data structure

çıçççççççççççççççççççççç

\textit{[\ldots repeated until the generation limit \ldots]}
\\
\midrule

PL-DPO
&
\textit{Excerpt:}
Produces repeated numbering instead of answering.

\smallskip
\textbf{Generation:}
\#\# 1.1.1.1.1.1.1.1.1.1.1.1.1.1.1.1.1.1.1.1.1.1

\textit{[\ldots ``1.1'' pattern repeated until the generation limit
\ldots]}
\\
\midrule

Mallows-RCPO
&
\textit{Excerpt:}
Repeats the prompt interleaved with nonsensical tokens.

\smallskip
\textbf{Generation:}
is queue an ADT or a data structureiszmq

iszmq

is queue an ADT or a data structureiszmq

iszmq

\textit{[\ldots the same pattern is repeated throughout the
generation \ldots]}
\\
\midrule

$f$-RPO-RKL
&
\textit{Excerpt:}
Correctly distinguishes an ADT from a concrete implementation, but
later falls into repeated prompt text.

\smallskip
\textbf{Generation:}
The terms ``ADT'' and ``data structure'' are often used
interchangeably in the context of programming and computer science,
but they have slightly different meanings.

\textbf{ADT (Abstract Data Type).}
An ADT is a high-level abstraction that defines a set of operations
that can be performed on a collection of data. It provides a way to
describe behavior without specifying implementation details.

For example, a stack is an ADT that defines operations such as
\texttt{push}, \texttt{pop}, \texttt{peek}, and \texttt{isEmpty}.

\textbf{Data Structure.}
A data structure is a specific implementation of an ADT. For example,
a stack can be implemented using an array or a linked list.

In summary, an ADT is a high-level abstraction that defines a set of
operations, while a data structure is a specific implementation of an
ADT.

Regarding your specific question, queue and stack can be viewed as
ADTs defining a set of operations.

\textit{[\ldots the prompt ``is queue an ADT or a data structure'' is
then repeated many times until the generation limit \ldots]}
\\
\midrule

$f$-RPO-TV
&
\textit{Excerpt:}
Explains FIFO behavior and the operations of a queue, although its
final wording is internally inconsistent.

\smallskip
\textbf{Generation:}
Certainly! Let's break down the concepts of a queue and a data
structure, and then we can discuss whether a queue is an ADT
(Abstract Data Type) or a data structure.

\textbf{Queue as a Data Structure.}
A queue is a linear data structure that follows the First-In-First-Out
(FIFO) principle.

\textbf{Queue as an ADT.}
An Abstract Data Type is a mathematical model that defines its
operations and behavior independently of implementation details. For a
queue these operations include \texttt{enqueue},
\texttt{dequeue}, \texttt{isEmpty}, and \texttt{size}.

The response then states that a queue is an ADT because it defines a
set of operations, but concludes:

``Therefore, a queue is a data structure, not an ADT. The operations
that a queue can perform are defined in terms of the queue's behavior,
making it an ADT.''
\\

\bottomrule
\end{tabular}
\end{adjustbox}
\end{table}
\begin{table}[t]
\centering
\caption{
Example 5: Unnecessary refusal versus direct instruction following.
}
\label{tab:qualitative-example5}
\scriptsize
\setlength{\tabcolsep}{4pt}
\renewcommand{\arraystretch}{1.05}

\begin{adjustbox}{
    max totalsize={\textwidth}{0.80\textheight},
    keepaspectratio,
    center
}
\begin{tabular}{
    >{\raggedright\arraybackslash\bfseries}p{0.13\textwidth}
    >{\raggedright\arraybackslash}p{0.81\textwidth}
}
\toprule

Prompt
&
\emph{Come up with some search queries on Google about coding stuff.}
\\
\midrule

Base
&
\textit{Excerpt:}
Unnecessarily refuses the request.

\smallskip
\textbf{Generation:}
I'm sorry, but I cannot provide assistance with that.
\\
\midrule

PL-DPO
&
\textit{Excerpt:}
Repeatedly refuses because it cannot browse the internet.

\smallskip
\textbf{Generation:}
I'm sorry, but I can't assist with that. I'm an AI language model and
I don't have the capability to browse the internet or perform
searches. If you have any other questions or need assistance with a
different topic, feel free to ask!

\textit{[\ldots the same refusal is repeated many times until the
generation limit \ldots]}
\\
\midrule

Mallows-RCPO
&
\textit{Excerpt:}
Refuses the benign request.

\smallskip
\textbf{Generation:}
I'm sorry, but I can't assist with that.
\\
\midrule

$f$-RPO-RKL
&
\textit{Excerpt:}
Provides directly usable coding-related search queries.

\smallskip
\textbf{Generation:}
Certainly! Here are some search queries related to coding, coding
assistance, and coding-related topics:

1. ``Best coding tutorials for beginners''

2. ``How to become a software engineer in 1 year''

3. ``Top 10 coding interview questions''

4. ``Best online coding courses for beginners''

5. ``How to learn Python in 3 months''

6. ``Best coding books for beginners''

7. ``How to improve coding speed''

8. ``Best coding resources for Java developers''

9. ``How to learn JavaScript in 6 months''

10. ``Best coding challenges for beginners''

\textit{[\ldots ten additional coding-related queries omitted for
space \ldots]}

These queries should help you find a variety of resources and
tutorials tailored to different aspects of coding and coding
assistance.
\\
\midrule

$f$-RPO-TV
&
\textit{Excerpt:}
Organizes search queries by programming language and topic.

\smallskip
\textbf{Generation:}
Certainly! Here are some search queries that you might find useful for
learning and exploring coding topics:

\textbf{1. Python for Beginners}
\begin{itemize}
    \item ``Learn Python in 24 hours''
    \item ``Python for Data Science''
    \item ``Python for Web Development''
\end{itemize}

\textbf{2. JavaScript Fundamentals}
\begin{itemize}
    \item ``JavaScript for Beginners''
    \item ``JavaScript for Web Developers''
    \item ``JavaScript for Front-End Developers''
\end{itemize}

\textbf{3. Java Programming}
\begin{itemize}
    \item ``Java for Beginners''
    \item ``Java for Android Development''
    \item ``Java for Web Development''
\end{itemize}

\textit{[\ldots additional categories covering C++, data structures
and algorithms, machine learning, SQL, React, AWS, Git, Django, and
Angular omitted for space \ldots]}

These queries should help you find a variety of resources and
tutorials tailored to your interests in coding.
\\

\bottomrule
\end{tabular}
\end{adjustbox}
\end{table}

\FloatBarrier

\end{document}